\documentclass{article}

\usepackage{PRIMEarxiv}
\makeatletter\newcommand{\appref}[1]{\@ifclassloaded{elsarticle}{\ref{#1}}{Appendix~\ref{#1}}}\makeatother
\usepackage[numbers]{natbib}

\usepackage[utf8]{inputenc} 
\usepackage[T1]{fontenc}    
\usepackage{hyperref}       
\usepackage{url}            
\usepackage{booktabs}       
\usepackage{amsfonts}       
\usepackage{nicefrac}       
\usepackage{microtype}      
\usepackage{lipsum}
\usepackage{fancyhdr}       
\usepackage{graphicx}       
\usepackage[export]{adjustbox}
\usepackage{amsmath}   
\usepackage{amssymb}   
\usepackage{amsfonts}  
\usepackage{subfigure}
\usepackage{algorithm}
\usepackage{algpseudocode}
\usepackage{amsthm}
\usepackage{hyperref}
\newtheorem{theorem}{Theorem}[section]

\newtheorem{proposition}[theorem]{Proposition}

\title{Adaptive diffusion posterior sampling for data and model fusion of complex nonlinear dynamical systems}
\author{
  Dibyajyoti Chakraborty \\
  College of Information Sciences and Technology,\\
  Pennsylvania State University,\\
  University Park, 16802, PA, USA
  \AND
  Hojin Kim, Romit Maulik \\
  School of Mechanical Engineering,\\
  Purdue University,\\
  West Lafayette, 47907, IN, USA
}

\date{December 2025}

\begin{document}

\maketitle
\begin{abstract}
High-fidelity numerical simulations of chaotic, high dimensional nonlinear dynamical systems are computationally expensive, necessitating the development of efficient surrogate models. Most surrogate models for such systems are deterministic, for example when neural operators are involved. However, deterministic models often fail to capture the intrinsic distributional uncertainty of chaotic systems. This work presents a surrogate modeling formulation that leverages generative machine learning, where a deep learning diffusion model is used to probabilistically forecast turbulent flows over long horizons. We show that the per-step error of a deterministic surrogate is irreducible for chaotic systems, justifying the use of probabilistic models. To handle complex, unstructured geometries, we utilize a multi-scale graph transformer architecture incorporating diffusion preconditioning and voxel-grid pooling. More importantly, our modeling framework provides a unified platform that also predicts spatiotemporally important locations for sensor placement, either via uncertainty estimates or through an error-estimation module. Finally, the observations of the ground truth state at these dynamically varying sensor locations are assimilated using diffusion posterior sampling requiring no retraining of the surrogate model. We present our methodology on two-dimensional homogeneous and isotropic turbulence and for a flow over a backward-facing step, demonstrating its utility in forecasting, adaptive sensor placement, and data assimilation for high dimensional chaotic systems.
\end{abstract}

\section{Introduction}\label{sec:intro}
High-fidelity simulation of chaotic dynamical systems such as turbulent flows remains prohibitively expensive at practically relevant Reynolds numbers, with DNS requiring $\mathcal{O}(Re^{9/4})$ degrees of freedom in three dimensions 
\cite{pope2000turbulent}. While deep learning surrogates have dramatically reduced this cost, deterministic models share a fundamental statistical deficiency: they collapse the full conditional distribution of the true stochastic system to 
a point prediction. For any chaotic system with intrinsic state variance, this incurs an irreducible per-step Wasserstein error that can accumulate exponentially, at a rate set by the surrogate's Lipschitz constant, over the forecast horizon. Diffusion models \cite{ho2020denoising,karras2022elucidating} resolve this by learning the full conditional distribution, making their per-step mismatch entirely reducible by training (Ref. Section \ref{sec:error_analysis} for further details). In several practical applications, training data from the dynamical system that may be used by a surrogate, is available on point clouds or meshes that result from numerical discretizations of governing laws. This brings forth the need to design accurate diffusion-based surrogate models that are specialized for forecasting high dimensional chaotic dynamical systems on unstructured meshes.  

Beyond forecasting, practical deployment requires fusing model predictions with sparse runtime observations from physical sensors. Diffusion posterior sampling \cite{chung2022diffusion,rozet2023score} achieves this by augmenting each denoising step with a likelihood gradient, steering the generative process toward the posterior without retraining. This raises a natural follow-on question: \textit{where} should sensors be placed to maximally reduce posterior uncertainty? In the context of Bayesian data assimilation, there exist several techniques for optimal sensor placement that are both physics and data-driven \cite{hinson2015observability,manohar2018data,ma2025physense}. Several approaches rely on information theoretic approaches such as A/D/E optimal design which are optimization-based strategies that rely on assumptions for the shape of the posterior distribution \cite{chaloner1995bayesian}. Observability-based approaches based on control theory choose (typically fixed) sensor locations to maximize a metric of the observability Gramian \cite{bopardikar2021randomized,brace2022sensor}. When forecasting with projection-based reduced-order models, QR pivoting is a common approach that leverages the identified proper orthogonal decomposition modes from training data \cite{manohar2018data}. Discrete empirical interpolation methods also identify sensor placement strategies that identify points which are important for nonlinear term reconstruction \cite{kim2025towards}. However, most approaches for sensor placement tend to be fixed or require expensive online computations which reduces their utility for real-time data and model fusion.  To address this gap, we propose two approaches for optimal sensor placement using a trained diffusion model. The first approach relies on the fact that a diffusion model can be sampled to generate ensembles of forecasts. A diffusion ensemble is itself a spatially-resolved uncertainty field, making it a direct and principled guide for sensor placement - closing the loop between inference and observation in a unified active inference framework. However, uncertainty guided sensor placement suffers from increased computational costs for ensemble generation via sampling the diffusion model. To complement this, we propose a metamodeling strategy that uses \emph{another} model to predict the rollout error of the surrogate and to use spatial concentrations of this predicted error to sample new points efficiently. 

Based on these, our work delivers four contributions : \textbf{(i)} An error analysis of autoregressive surrogates (Section~\ref{sec:error_analysis}) that motivates a probabilistic surrogate for chaotic flows. \textbf{(ii)} A multi-scale Graph Transformer diffusion architecture combining Elucidating Diffusion Model (EDM) preconditioning \cite{karras2022elucidating}, AdaLN-Zero conditioning \cite{peebles2023scalable}, and hierarchical voxel-grid pooling to forecast in unstructured meshes (this graph architecture is only used for the backward-facing step case which has unstructured grid. \textbf{(iii)} Two adaptive sensor placement strategies: a learned predictive error network based estimator and a diffusion ensemble-variance-based estimator targeting observations where the generative prior has most uncertainty. \textbf{(iv)} A topology-aware greedy selection algorithm with spatial suppression that enforces a minimum separation between sensors, so that redundant clustering in localized high-uncertainty regions is reduced (Figure~\ref{fig:bfs_plots/different_sampling}). We validate the framework on two-dimensional forced homogeneous isotropic turbulence at $Re = 1000$  and flow over a backward-facing step at $Re = 26{,}000$ on an unstructured high-dimensional finite-element mesh. The associated codes are available at \href{https://github.com/ISCLPurdue/chaos_gen}{https://github.com/ISCLPurdue/chaos\_gen}.

\subsection{Related Works}

Surrogate modeling replaces expensive high-fidelity simulations with computationally efficient approximation that reproduce the input-output behavior of a target system.  Classical methods such as Gaussian processes\cite{williams2006gaussian}, radial basis functions \cite{buhmann2000radial}, and polynomial response surfaces \cite{sacks1989design,forrester2008engineering} provide well-calibrated uncertainty estimates for low-dimensional systems, but their scalability rapidly degrades with the dimensionality of the problem \cite{razavi2012review}. Deep neural operator architectures have substantially extended the scope of tractable surrogate tasks: the Fourier Neural Operator \cite{li2020fourier} learns resolution-invariant solution
operators for parametric PDEs, while DeepONet \cite{lu2021learning} exploits the universal approximation theorem for operators to generalize across functional inputs. For time-dependent systems, recurrent architectures such as LSTM \cite{hochreiter1997long} and gated recurrent units \cite{cho2014learning} capture long-range temporal dependencies and have been applied to data-driven forecasting of high-dimensional chaotic attractors \cite{vlachas2018data,vlachas2020backpropagation}.  Specialized surrogate models\cite{geneva2022transformers,wu2024transolver,solera2024beta,li2023scalable} have further demonstrated competitive performance in reduced-order modeling of chaotic systems such as turbulent flows. Despite these advances, deterministic surrogates produce only point predictions and cannot represent the distributional uncertainty intrinsic to stochastic or chaotic dynamics. Diffusion models \cite{sohl2015deep,ho2020denoising,song2020score,karras2022elucidating} resolve this limitation by learning a reverse-time stochastic differential equation that maps Gaussian noise to complex high-dimensional distributions. Their superior mode coverage over GANs \cite{dhariwal2021diffusion} and natural extension to video and spatiotemporal sequences \cite{ho2022video} make them especially attractive for dynamical system surrogates. Some diffusion-inspired surrogate models, such as PDE-Refiner \cite{lippe2023pde}, further demonstrate that an iterative refinement process improves long-rollout and spectral accuracy relative to one-step neural solvers. PDE-Refiner refines a deterministic one-step prediction with a few noise-conditioned denoising passes and is trained on ground-truth conditioning. PAINT \cite{radler2025paint} instead generates a whole time window jointly from sparse probe measurements, without an autoregressive recurrence. Our surrogate keeps the autoregressive recurrence; it could be extended to windows of time steps, as in video diffusion models \cite{ho2022video}, but for very long horizon forecasts of high-dimensional systems an autoregressive model is more feasible. Moreover, experiments in previous work have shown the advantages of applying generative diffusion models over traditional ML surrogate models for a variety of high dimensional dynamical systems \cite{gao2024generative,kohl2026benchmarking,oommen2024integrating,liu2024uncertainty,lienen2023zero,rasul2021autoregressive,gao2024bayesian,finn2024generative,yang2023denoising,molinaro2024generative}. In this work, we perform a theoretical comparison of long term error propagation for deterministic and generative surrogate models. 

Unstructured meshes and irregular topologies that arise naturally in computational fluid dynamics and structural mechanics have been tackled by Graph neural networks (GNNs) based surrogate models.  \citet{sanchez2020learning} established a foundational encoder-processor-decoder GNN capable of simulating complex multi-physics on particle and mesh representations.  \citet{pfaff2020learning} specialized this architecture to finite-element meshes, demonstrating orders-of-magnitude speed-up over high-fidelity solvers.  Message-passing neural PDE
solvers \cite{brandstetter2022message} and multi-scale GNN variants \cite{lino2022towards} further improve accuracy on flows with different spatial scales. Another direction of work shows that the combination of differentiable PDE solvers with GNN surrogates
\cite{belbute2020combining,kim2023generalizable} yields physically consistent predictions on
complex problems. Furthermore, GraphCast \cite{lam2022graphcast} uses a multi-mesh message-passing graph neural network to achieve state-of-the-art global weather prediction, underscoring the scalability of graph-based surrogates. However, as in the previous case, deterministic GNN surrogates produce only point predictions and cannot capture the distributional spread of stochastic or turbulent states. Graph diffusion models have been shown to be successful in several specific tasks such as molecule synthesis \cite{liu2024graph,wang2023graph,kim2024diffusion}, protein folding \cite{yi2023graph}, weather forecasting \cite{price2023gencast}, and traffic and air quality prediction \cite{wen2023diffstg}. Sparse polynomial surrogates for diffusion processes on graphs \cite{d2025surrogate} provide convergence guarantees for least squares and compressed sensing recovery. Most directly relevant to this work, \citet{lino2025learning} proposed the Diffusion Graph Network (DGN), a latent diffusion model with a multi-scale GNN processor that samples equilibrium flow field distributions on unstructured meshes, and its SE(3)-equivariant predecessor \cite{valencia2024se} demonstrated that incorporating geometric symmetry into the score network significantly improves sample fidelity.

Probabilistic surrogates naturally interface with data-assimilation pipelines \cite{evensen2003ensemble, bocquet2019data, buizza2022data}, motivating the posterior-sampling perspective adopted here. Within this framework, the question of \textit{where} to observe is formalized by optimal experimental design (OED), which optimizes sensor placement with respect to information gain \cite{chaloner1995bayesian, ryan2016review}. For high-dimensional PDE systems, tractable placement strategies have been developed through sparse selection on POD bases, connecting compressed sensing to structured sensor placement \cite{manohar2018data}. Active learning extends this idea sequentially: as the surrogate improves, placement is iteratively refined by targeting regions of high posterior uncertainty \cite{settles2009active, gal2017deep}, directly motivating the adaptive strategy used here. Crucially, decoupling sensor placement from inference is suboptimal; co-optimizing sensor layouts with the inference procedure yields substantially tighter posteriors across a range of inverse problems \cite{siahkoohi2022deep, blanchard2021bayesian}. Separately, diffusion posterior sampling \cite{chung2022diffusion, rozet2023score} has emerged as a powerful reconstruction framework conditioning the reverse diffusion process on sparse observations. This paradigm has shown promise in sparse reconstruction and model-data fusion \cite{oommen2025learning, parikh2026d, chakraborty2025multimodal,huang2024diffusionpde}. Recent physics-informed diffusion models further embed PDE residuals as constraints during training \cite{bastek2024physics} or sampling \cite{zheng2025inversebench,jacobsen2025cocogen}, enabling data-efficient learning of physical systems and directly underpinning the adaptive diffusion posterior sampling approach proposed in this work. Furthermore, rather than treating sensor placement as a pre-processing step, a diffusion model enables uncertainty from the generative model to \textit{inform} placement, integrating the two into a unified inference loop, as shown further in this work. 

Some recent work has explored adaptive sensor placement for physical systems using generative and learning-based methods. \citet{zhao2025generative} and \citet{son2025diffusion} leverage graph diffusion policies and diffusion-attention mechanisms for adaptive combinatorial sensor placement in power systems and distributed parameter systems, respectively. \citet{vishwasrao2025diff} introduced Diff-SPORT, combining diffusion posterior sampling with Shapley-value attribution to rank and select sensor locations for urban turbulent flow reconstruction, although modeling temporal dynamics and out-of-distribution generalization remain open. \citet{ma2025physense} proposed PhySense, coupling a flow-based generative model with projected gradient descent to co-optimize reconstruction and sensor placement with theoretical guarantees. \citet{wang2023learning} trained a Transformer policy through actor-critic reinforcement learning to iteratively reallocate climate sensors, though without a generative prior for posterior-based field reconstruction. The costs of such added optimization and applications to high dimensional dynamical systems remain unexplored. In contrast, our work unifies autoregressive probabilistic forecasting, diffusion posterior sampling, and a topology-aware adaptive sensor placement in a single closed-loop during inference for high dimensional chaotic systems on unstructured meshes.

\section{Methods}
\label{sec:methods}

\subsection{Diffusion-based Generative Modeling}
Diffusion models function as generative priors by reversing a forward stochastic process that gradually degrades data structure until it resembles pure noise. We adopt the Elucidating Diffusion Model (EDM) framework \cite{karras2022elucidating}, which uses a "variance exploding" forward process. For a clean data sample $\mathbf{x}_0$ and a noise level $\sigma > 0$, the noising process is defined as:
\begin{equation}
    \mathbf{x}(\sigma) = \mathbf{x}_0 + \sigma \mathbf{n}, \quad \text{where} \quad \mathbf{n} \sim \mathcal{N}(\mathbf{0}, \mathbf{I}).
\end{equation}
During training, instead of directly predicting added noise, a preconditioned \cite{karras2022elucidating} neural network $F_{\theta}$ is trained to output the denoised data estimate $D_{\theta}(\mathbf{x}(\sigma); \sigma) \approx \mathbf{x}_0$ using a weighted mean squared error over the noise schedule:
\begin{equation}
    \mathcal{L}_{\mathrm{EDM}}(\theta) = \mathbb{E}_{\mathbf{x}_0, \mathbf{n}, \sigma} \left[ \lambda(\sigma) \| \mathbf{x}_0 - D_{\theta}(\mathbf{x}(\sigma); \sigma) \|^2 \right],
\label{eq:edm_loss}
\end{equation}
where the weighting function $\lambda(\sigma)$ ensures balance between different levels of noise.

The denoiser is related to the raw network output $F_\theta$ by the EDM preconditioning
\begin{equation}\label{eq:precond}
    D_\theta(\mathbf{x};\sigma) =
    c_{\mathrm{skip}}(\sigma)\,\mathbf{x}
    + c_{\mathrm{out}}(\sigma)\,
      F_\theta\!\bigl(c_{\mathrm{in}}(\sigma)\,\mathbf{x};\,
                      c_{\mathrm{noise}}(\sigma)\bigr),
\end{equation}
with $c_{\mathrm{skip}} = \sigma_d^2/(\sigma^2+\sigma_d^2)$, $c_{\mathrm{out}} = \sigma\sigma_d/\sqrt{\sigma^2+\sigma_d^2}$, $c_{\mathrm{in}} = 1/\sqrt{\sigma^2+\sigma_d^2}$ and $c_{\mathrm{noise}} = \tfrac{1}{4}\ln\sigma$, where $\sigma_d$ denotes the data standard deviation ($\sigma_d=1$ for our standardized fields). The loss weight is $\lambda(\sigma) = (\sigma^2 + \sigma_d^2)/(\sigma\,\sigma_d)^2$, which up-weights low-noise predictions, and the noise levels are drawn per sample from a log-normal distribution, $\ln\sigma \sim \mathcal{N}(P_\mu, P_\sigma^2)$ (values in Table~\ref{tab:combined_hyperparams}). The squared norm in Equation~\eqref{eq:edm_loss} is averaged over all nodes (or grid points) and channels.

Once trained, $D_{\theta}$ yields the score of the noised data distribution, $\nabla_{\mathbf{x}} \log p(\mathbf{x}(\sigma)) = \bigl(D_{\theta}(\mathbf{x}(\sigma);\sigma)-\mathbf{x}(\sigma)\bigr)/\sigma^2$, which points towards the data manifold. This allows for deterministic sampling via the Probability Flow ODE \cite{song2020score, karras2022elucidating}:
\begin{equation}
    d\mathbf{x} = \frac{\mathbf{x}(\sigma) - D_{\theta}(\mathbf{x}(\sigma); \sigma)}{\sigma}\, d\sigma,
\label{eq:ode}
\end{equation}
i.e.\ $d\mathbf{x}/d\sigma = -\sigma\,\nabla_{\mathbf{x}}\log p(\mathbf{x}(\sigma))$. Generation is performed by initializing $\mathbf{x}(\sigma_{\text{max}}) \sim \mathcal{N}(\mathbf{0}, \sigma_{\text{max}}^2 \mathbf{I})$ at a large noise level ($\sigma_{\text{max}}$) and integrating this ODE from $\sigma_{\max}$ down to $\sigma = 0$.

\subsection{\texorpdfstring{Training and Autoregressive Forecasting}{Training and Autoregressive Forecasting}}\label{sec:training}

In the single-step training recipe, the denoiser is conditioned on the current state $\mathbf{x}^{(0)}$ and trained similar to Equation~\eqref{eq:edm_loss} to denoise the next state, $D_\theta(\mathbf{x}^{(1)}+\sigma\mathbf{n};\,\sigma,\,\mathbf{x}^{(0)})\approx\mathbf{x}^{(1)}$. The ground-truth pairs $(\mathbf{x}^{(0)},\mathbf{x}^{(1)})$ are drawn from the ground truth trajectories. At inference the forecast is produced autoregressively: the sampler of Equation~\eqref{eq:ode} is run from pure noise conditioned on the current state, the resulting sample becomes the conditioning for the next step, and so on. It can also be trained with a multi-step variant of this objective, in which the model is conditioned on its own forecasts. \appref{sec:multistep_app} describes it and compares it with single-step training. In our use cases, we did not see any significant improvement in performance by training with a multi-step variant of this objective. However, the BFS results use the multi-step variant due to its slight advantage.

\subsection{Conditional Generation}
To perform conditional generation given sparse or noisy measurements $\mathbf{y}$, we aim to sample from the posterior distribution $p(\mathbf{x}|\mathbf{y})$. For fixed conditions, for example the previous history of states in a forecasting problem, they can be directly inputted to the network $D_\theta$ as its conditioning argument. However, for test time conditions like sparse observations, we employ a posterior sampling strategy. By applying Bayes' rule to the score function, posterior score can be divided into a prior term and a likelihood guidance term:
\begin{equation}
    \nabla_{\mathbf{x}(\sigma)} \log p(\mathbf{x}(\sigma)|\mathbf{y}) = \underbrace{\nabla_{\mathbf{x}(\sigma)} \log p(\mathbf{x}(\sigma))}_{\text{Prior (Model)}} + \underbrace{\nabla_{\mathbf{x}(\sigma)} \log p(\mathbf{y}|\mathbf{x}(\sigma))}_{\text{Likelihood (Guidance)}}.
\label{eq:posterior_score}
\end{equation}
Directly computing the likelihood $p(\mathbf{y}|\mathbf{x}(\sigma))$ is intractable because $\mathbf{x}(\sigma)$ is a noisy latent variable. Methods such as Diffusion Posterior Sampling (DPS) \cite{chung2022diffusion} approximate this term by estimating the clean data $\hat{\mathbf{x}}_0$ using Tweedie's formula, which is equivalent to the denoiser output: $\hat{\mathbf{x}}_0 \approx D_{\theta}(\mathbf{x}(\sigma); \sigma)$ in our case.

We adopt the Score-based Data Assimilation (SDA) approach \cite{rozet2023score}, which improves upon standard DPS by normalizing the gradients with a time-dependent variance estimate. This accounts for the diminishing information content of the measurements at high noise levels. The approximate likelihood gradient used to guide the solver is:
\begin{equation}
    \nabla_{\mathbf{x}(\sigma)} \log p(\mathbf{y}|\mathbf{x}(\sigma)) \approx - \nabla_{\mathbf{x}(\sigma)} \frac{\| \mathbf{y} - M(D_\theta(\mathbf{x}(\sigma); \sigma)) \|^2}{2(\Sigma_y + \sigma^2\hat{\Gamma})},
\label{eq:likelihood_score_edm}
\end{equation}
where $M$ is the measurement operator, $\Sigma_y$ is the measurement noise variance, and $\hat{\Gamma}$ is a scalar approximating the Jacobian-prior variance. In both test cases, unconditional and guided sampling use the deterministic second-order Heun sampler of EDM with 50 steps. A few sampling hyperparameters can be tuned to get the optimal sampling performance. For example, the likelihood gradient is clipped element-wise to $[-c_{\mathrm{clip}}, c_{\mathrm{clip}}]$, and multiplied by a likelihood scale $\beta$ and by the noise level $\sigma_i$ of sampler step $i$. At the high noise levels of the early sampler steps, $\sigma^2\hat\Gamma$ exceeds $\Sigma_y$ and dominates the denominator, so $\hat\Gamma$ damps the guidance there. Together, $\Sigma_y=\sigma_{\mathrm{obs}}^2$, $\hat\Gamma$, $\beta$ and $c_{\mathrm{clip}}$ set the guidance strength (values in Table~\ref{tab:combined_hyperparams}).

\subsection{Ensemble Kalman filter}\label{sec:enkf}
We also assimilate the observations with the stochastic ensemble Kalman filter (EnKF) \cite{evensen2003ensemble,burgers1998analysis}, applied to the same diffusion forecast. At every step each member $e=1,\dots,E$ is advanced by the unconditional sampler to a forecast $\mathbf{x}^{f}_e$, and the ensemble is then corrected with the perturbed-observation analysis
\begin{equation}\label{eq:enkf}
    \mathbf{x}^{a}_e=\mathbf{x}^{f}_e+\mathbf{K}\bigl(\mathbf{y}+\boldsymbol{\eta}_e-M\mathbf{x}^{f}_e\bigr),\qquad
    \mathbf{K}=\mathbf{P}^{f}M^\top\bigl(M\mathbf{P}^{f}M^\top+\mathbf{R}\bigr)^{-1},\qquad
    \boldsymbol{\eta}_e\sim\mathcal{N}(\mathbf{0},\mathbf{R}),
\end{equation}
where $\mathbf{R}=\sigma_{\mathrm{obs}}^2\mathbf{I}$ and $\mathbf{P}^{f}$ is the sample covariance of the forecast ensemble, updated with the Gaspari-Cohn function \cite{gaspari1999construction} (covariance localization \cite{houtekamer2001sequential}) because $E$ is much smaller than the state dimension. The perturbed observations $\boldsymbol{\eta}_e$ keep the spread of the analysis ensemble consistent with the Kalman update \cite{burgers1998analysis}. The EnKF thus differs from posterior sampling only in the update: it corrects the sampled forecast linearly, using the ensemble covariance and one sensor layout shared by all members, whereas posterior sampling guides each member separately through the likelihood score of Equation~\eqref{eq:likelihood_score_edm}. We use it for the backward-facing step (Section~\ref{sec:bfs_baselines}; settings in \appref{sec:bfs_enkf}) as a reference for comparison. We acknowledge that there are other advanced techniques of traditional data assimilation, which we leave for future work.

\subsection{Graph Diffusion Architecture}\label{sec:graph_diffusion_arch}

We define a hierarchical graph network $F_\theta(\mathbf{x}(\sigma); \sigma, \mathbf{x}_{\mathrm{cond}})$, where the noisy state $\mathbf{x}(\sigma)$ consists of node features $\mathbf{X} \in \mathbb{R}^{N \times C}$, $C$ being the number of physical channels per node ($C=2$ velocity components). The conditioning state $\mathbf{x}_{\mathrm{cond}}$ (the current state) is concatenated with the scaled noisy state along the channels, so the network input has $2C$ channels per node. The architecture follows a U-Net design with symmetric encoder and decoder paths connected by skip connections, using a backbone of Graph Diffusion Transformer blocks built on GATv2 attention \cite{brody2022attentive} (GDT blocks in Figure~\ref{fig:gdt_architecture}) conditioned on the noise level $\sigma$. 

\begin{figure}[!htp]
    \centering
    \includegraphics[width=\linewidth]{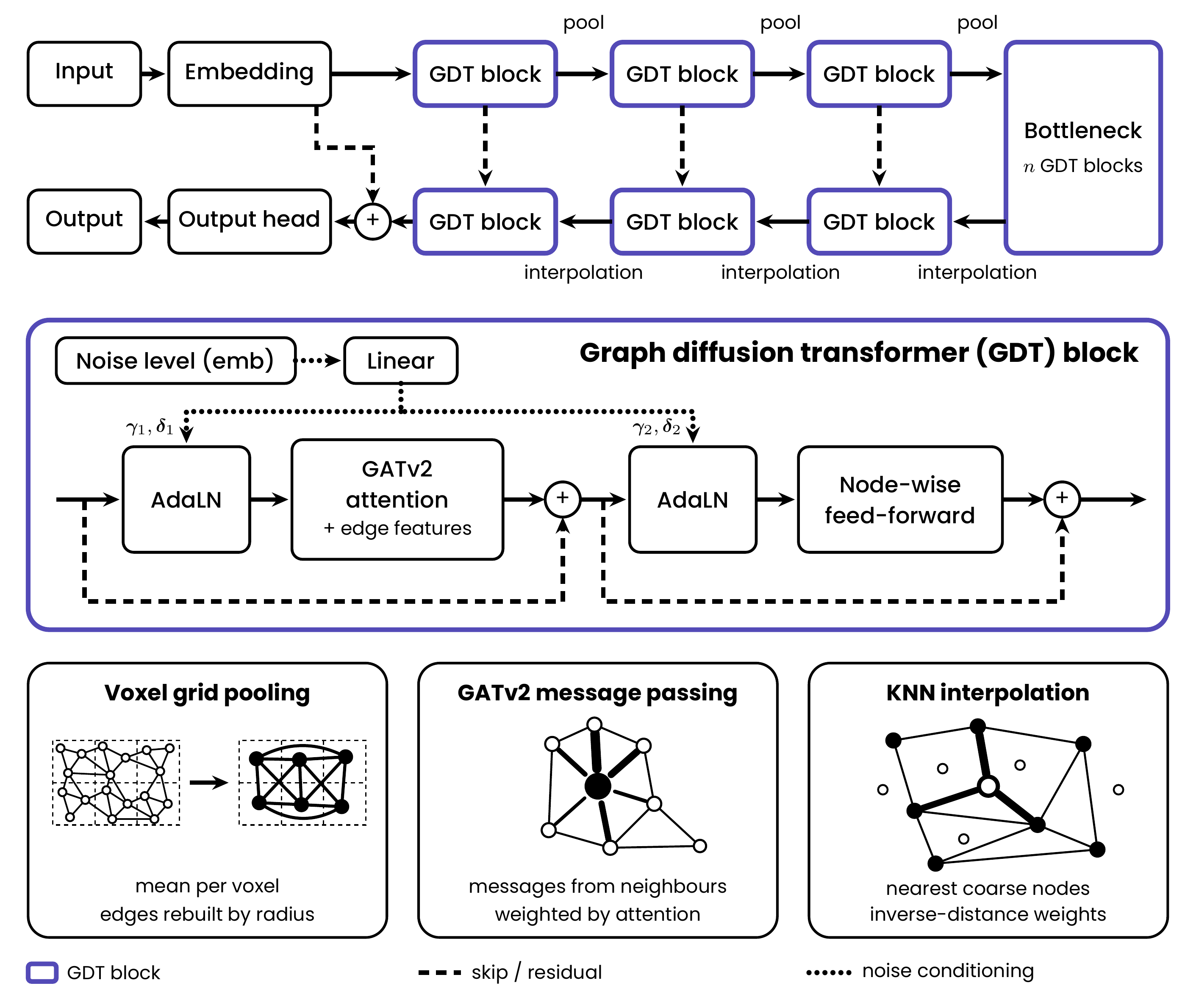}
    \caption{Graph diffusion architecture. Top: the U-Net on the mesh graph, with graph diffusion transformer (GDT) blocks and voxel-grid pooling in the encoder, $n=2$ GDT blocks at the coarsest level (bottleneck), and KNN interpolation with GDT blocks in the decoder. Middle: one GDT block (AdaLN-modulated GATv2 attention with edge features and a node-wise feed-forward layer, each with a residual connection); the noise-level embedding sets the AdaLN parameters $\boldsymbol{\gamma}$ and $\boldsymbol{\delta}$. Bottom: the three graph operations.}
    \label{fig:gdt_architecture}
\end{figure}

\paragraph{Input Embedding and Noise Conditioning.}
Continuous node positions are lifted to a higher-dimensional representation via Gaussian Fourier projections,
\begin{equation}
    \mathbf{h}_{\mathrm{pos}} = \bigl[\sin(2\pi \mathbf{W}\mathbf{p})
    \;\|\; \cos(2\pi \mathbf{W}\mathbf{p})\bigr],
\end{equation}
where $\mathbf{W} \sim \mathcal{N}(0, \kappa^2)$ is a fixed random matrix and $\mathbf{p} \in \mathbb{R}^d$ is a node position. We use $\kappa=30$\,m$^{-1}$, i.e.\ a length scale $1/\kappa\approx 3.3$\,cm (about 2.6 step heights), fixed a priori and not tuned. The concatenation $[\mathbf{X} \| \mathbf{h}_{\mathrm{pos}}]$ is projected to the hidden dimension $d_h$ by a linear layer. Edge attributes are constructed from relative positions and Euclidean distances between connected nodes and projected to $\mathbb{R}^{d_h}$.

The noise level enters through $c_{\mathrm{noise}}(\sigma)=\tfrac14\ln\sigma$, which is encoded via a sinusoidal embedding followed by an MLP to produce a conditioning vector $\mathbf{e}_\sigma \in \mathbb{R}^{d_t}$, where $d_t$ is the embedding width. This is injected into each block via Adaptive Layer Normalization with zero initialization (AdaLN-Zero):
\begin{equation}
    \mathrm{AdaLN}(\mathbf{h},\, \mathbf{e}_\sigma)
    = \bigl(1 + \boldsymbol{\gamma}(\mathbf{e}_\sigma)\bigr)
      \odot \mathrm{LayerNorm}(\mathbf{h})
      + \boldsymbol{\delta}(\mathbf{e}_\sigma),
\end{equation}
where $\boldsymbol{\gamma}$ and $\boldsymbol{\delta}$ are linear projections initialized to zero. The $(1 + \boldsymbol{\gamma})$ formulation ensures that the noise-level modulation reduces to a plain LayerNorm at the start of training. Unlike in the Diffusion Transformer (DiT) \cite{peebles2023scalable}, there is no zero-initialized output gate.

\paragraph{Graph Diffusion Transformer Block.}
Between the encoder and decoder, two GDT blocks are applied at the coarsest resolution to allow global information exchange before upsampling begins. Each block applies an AdaLN-conditioned GATv2 attention convolution followed by a position-wise feed-forward network (FFN), both with residual connections:
\begin{equation}
    \label{eq:gnn_update}\mathbf{h}_i^{(l+1)} = \mathbf{h}_i^{(l)}
    + \sum_{j \in \mathcal{N}(i)} \alpha_{ij}\, \mathbf{W}_t\, \tilde{\mathbf{h}}_j^{(l)}, \qquad \tilde{\mathbf{h}}^{(l)} = \mathrm{AdaLN}(\mathbf{h}^{(l)},\mathbf{e}_\sigma),
\end{equation}
where the attention weights $\alpha_{ij}$ are computed from node and edge features. We use the GATv2 operator \cite{brody2022attentive} with $H=4$ heads: for each head,
\begin{equation}
    \alpha_{ij} = \operatorname{softmax}_{j\in\mathcal{N}(i)}\Bigl(\mathbf{a}^\top \mathrm{LeakyReLU}\bigl(\mathbf{W}_s\tilde{\mathbf{h}}_i + \mathbf{W}_t\tilde{\mathbf{h}}_j + \mathbf{W}_e\mathbf{e}_{ij}\bigr)\Bigr),
\end{equation}
where $\mathbf{e}_{ij}$ is the projected edge attribute and $\mathbf{a}$, $\mathbf{W}_s$, $\mathbf{W}_t$, $\mathbf{W}_e$ are learned; GATv2 has no separate value matrix, so $\mathbf{W}_t$ also transforms the aggregated features in Equation~\eqref{eq:gnn_update}. The outputs of the $H$ heads are concatenated. The FFN is applied after a second AdaLN modulation step, also with a residual.

\paragraph{Hierarchical Pooling and Unpooling.}
To capture multi-scale geometric structure, the encoder progressively coarsens the graph through voxel grid pooling, and the decoder restores resolution via distance-weighted interpolation.

\textit{Voxel grid pooling (downsampling).}
At each encoder level $\ell$ with voxel resolution $r_\ell$, nodes are assigned to spatial voxels and aggregated by mean pooling within each cluster $\mathcal{C}$:
\begin{equation}
    \mathbf{p}_{\mathrm{coarse}} = \frac{1}{|\mathcal{C}|}
    \sum_{j \in \mathcal{C}} \mathbf{p}_j, \qquad
    \mathbf{h}_{\mathrm{coarse}} = \frac{1}{|\mathcal{C}|}
    \sum_{j \in \mathcal{C}} \mathbf{h}_j.
\end{equation}
Edges on the coarsened graph are rebuilt by connecting all node pairs within a radius of $r_\ell\sqrt{2} \cdot 1.5$. Here $\sqrt{2}\,r_\ell$ is the distance between diagonally adjacent voxel centres, and the factor 1.5 is a margin because pooled positions are voxel means; a coarse node is thus typically connected to its eight neighbouring voxels and to some next-nearest ones.

\textit{KNN interpolation (upsampling).}
Features are propagated from coarse to fine resolution using inverse-distance weighting over the $k$ nearest coarse neighbors $\mathcal{N}_k(\mathbf{p}_f)$ of each fine node ($k=3$):
\begin{equation}
    \mathbf{h}_f = \sum_{c \in \mathcal{N}_k(\mathbf{p}_f)}
    \frac{w_c}{\sum_m w_m}\,\mathbf{h}_c, \qquad
    w_c = \|\mathbf{p}_f - \mathbf{p}_c\|_2^{-1}.
\end{equation}
The interpolated features are concatenated with the corresponding encoder skip features and fused by a linear projection before the subsequent Graph Transformer block.

\paragraph{Output and Preconditioning.}
A global long-skip connection adds the projected input to the final decoder output, preserving low-frequency content. The result is passed through Layer Normalization and a linear output projection to produce the denoised node features $\hat{\mathbf{x}}$. The raw network output $F_\theta$ is then wrapped with the noise-level-dependent EDM preconditioning \cite{karras2022elucidating} of Equation~\eqref{eq:precond}, which rescales inputs and outputs so the network operates in a unit-variance regime across all noise levels. This preconditioning helps to mitigate some of the limitations of a variance exploding diffusion model.

\subsection{Sensor Placement}\label{sec:sensor_placement}
Optimal sensor placement is critical for maximizing information gain during data assimilation. To condition the posterior diffusion generation on measurements, we implement three distinct sensor placement strategies. Each strategy selects a fixed budget of $s$ sensors. On the mesh, all layouts exclude the nodes on the top and bottom walls, found by a coordinate bounding box, and all layouts except the random one also exclude every node within $g$ graph hops of them, $g$ being the suppression gap defined below (\appref{sec:boundary_band}). The periodic turbulence grid has no boundary. These are chosen heuristically and must be updated based on the experiment and feasibility of sensor placements.

\paragraph{Random Placement}
Under the random sensor placement strategy, sensor locations are sampled uniformly at random from the pool of eligible interior points. This serves as a baseline for comparing other sensor placement strategies.

\paragraph{Predictive Model-based Placement}
We utilize a learned, adaptive strategy that positions sensors in spatial "hot spots" where the pre-trained diffusion model $D_{\theta}$ exhibits the highest reconstruction error from the prior, meaning it is least capable of generating the correct values without observation.

To quantify this, we define a ground-truth error field, $\mathcal{E}_{gt}$, as the expected squared difference between the ground truth and the  diffusion forecast from a single denoiser call at the largest noise level. This is defined at each node. Since in our diffusion framework, the network approximates the expected forecast when denoising at $\sigma_{\max}=80$ (where the noised next state carries almost no information about $\mathbf{x}_0$), this spatial error map is computed  as:
\begin{equation}\label{eq:pred_network}
\mathcal{E}_{gt,i} = \frac{1}{C}\,\mathbb{E}_{\mathbf{n}}\!\left[
\left\| \mathbf{x}_{0,i} - D_{\theta}(\mathbf{x}(\sigma_{\max}); \sigma_{\max},\, \mathbf{x}_{\text{cond}})_i \right\|_2^2\right],
\end{equation}
where the norm is computed over the $C$ channels only, and the expectation over $\mathbf{n}$ is estimated from five noise draws to reduce label noise. Since $\mathbf{x}_0$ is unknown at inference time, we train a secondary predictor network, $P_{\phi}$, to estimate this error field from the conditioning state alone, without sensor values. During training, $\mathbf{x}_{\text{cond}}$ is the true current state $\mathbf{x}^{(t)}$ (teacher forcing) and the target is $\mathcal{E}_{gt}$ for the next state, $\mathbf{x}_0=\mathbf{x}^{(t+1)}$, computed with the frozen diffusion model. Because only the spatial ranking of the error matters for placement, while its magnitude varies by two orders of magnitude between training snapshots, the target is standardized to zero mean and unit variance over the nodes of each snapshot, giving $\tilde{\mathcal{E}}_{gt}$. The predictor $P_{\phi}$ is a smaller network without noise-level conditioning (\appref{sec:hyperparams}), and minimizes the mean squared error against the standardized ground truth error:
\begin{equation}\label{eq:pred_network_label}
    \mathcal{L}(\phi) = \mathbb{E}\!\left[\bigl\| P_{\phi}(\mathbf{x}_{\text{cond}}) - \tilde{\mathcal{E}}_{gt} \bigr\|_2^2\right].
\end{equation}
During inference, one predicted error map $\hat{\mathcal{E}} = P_{\phi}(\mathbf{x}_{\text{cond}})$ is computed for each member at every forecast step from that member's current state, at a small fraction of the cost of one denoiser evaluation, and guides the observation points $\mathbf{S}$ construction via a topology-aware greedy strategy:
\begin{enumerate}
    \item \textbf{Initialization:} Define the set of candidate spatial locations, excluding any defined band as above.
    \item \textbf{Greedy Selection:} Iteratively select the location $i^*$ with the highest predicted error: $i^* = \text{argmax}_{i} (\hat{\mathcal{E}}_i)$.
    \item \textbf{Topology-aware Suppression:} To ensure broad spatial coverage and prevent sensor clustering, suppress the error values of all candidate locations within a defined local neighborhood (e.g., a spatial radius or structural adjacency distance) of $i^*$. The neighbourhood is set by the gap $g$. This is all points within a Euclidean distance of $g$ grid spacings of a point $i^*$ on a grid. For an unstructured mesh, it is all nodes within $g$ graph hops of $i^*$.
    \item \textbf{Iteration:} Repeat steps 2 and 3 until the measurement budget is exhausted.
\end{enumerate}

\paragraph{Uncertainty-Driven Placement}
Rather than relying on a parameterized error network, the uncertainty-driven strategy leverages the inherent variance of an active diffusion ensemble. This ensemble is the forecast ensemble itself. Let $\hat{\mathbf{x}}^{(e)}$ denote the state of member $e$ after step $t$. We calculate the local predictive uncertainty $\varsigma_i$ as the mean standard deviation across the $C$ feature channels:
\begin{equation}\label{uq_sensor_placement}
    \varsigma_i = \frac{1}{C} \sum_{c=1}^{C} \text{Std}\left(\{\hat{x}_{i,c}^{(1)}, \hat{x}_{i,c}^{(2)}, \dots, \hat{x}_{i,c}^{(E)}\}\right)
\end{equation}
Sensors are then placed at the locations exhibiting the highest uncertainty $\varsigma_i$, using the identical greedy selection algorithm and spatial suppression applied in the predictive method. Here, the spread after step $t$ sets the sensors of step $t+1$. At the first step, before a spread exists, the layouts of the random strategy can be used. Although this method does not have any initial cost like training the error model, it requires an ensemble to be generated at every step to estimate the sensor locations. However, this is not a problem for our work as the whole motivation of a generative model is to generate an ensemble, and not just a single sample.

\section{Error Analysis}
\subsection{Autoregressive Surrogates for Stochastic Dynamical Systems}\label{sec:error_analysis}

We show that the upper bound on the autoregressive error of a surrogate model depends on the per-step error, which is non-reducible for a deterministic surrogate, and an exponential growth part which depends on the surrogate's underlying smoothness. 

\paragraph{Setup.}
An interesting property of chaotic systems is that although having a deterministic character, they are often treated as a stochastic system due to their extreme sensitivity to initial conditions \cite{deco1997determining}. Note also, that stochasticity may also arise from unresolved physics, partial observability and model-form uncertainty. The observed or measured states may then be treated as samples from an evolving probability density function \cite{nie2015reconstruction, botvinick2025measure}. Let the true system evolve according to a stochastic Markov kernel $\mathbf{x}_t \sim \mathcal{T}(\cdot \mid \mathbf{x}_{t-1})$, where $\mathcal{T}(\cdot\mid\mathbf{x}) \in \mathcal{P}_2(\mathbb{R}^n)$ is a Markov kernel with finite second moments (no assumption is made on its shape or modality). Let $\mu_t$ denote the marginal of $\mathbf{x}_t$. Since the ground truth comes from a distribution, we measure the prediction error using the 2-Wasserstein distance $\mathcal{E}_t = W_2(\mu_t, \hat{\mu}_t)$, where $\hat{\mu}_t$ is the surrogate marginal at step $t$. For two probability measures $\mu, \nu \in \mathcal{P}_2(\mathbb{R}^n)$ with finite second moments, the 2-Wasserstein distance is defined as
\begin{equation}\label{eq:w2_def}
    W_2(\mu, \nu) \;=\; \left( \inf_{\pi \in \Pi(\mu, \nu)} \int_{\mathbb{R}^n \times \mathbb{R}^n} \|\mathbf{x} - \mathbf{y}\|^2 \, \mathrm{d}\pi(\mathbf{x}, \mathbf{y}) \right)^{1/2},
\end{equation}
where $\Pi(\mu, \nu)$ denotes the collection of all joint probability measures (couplings) on $\mathbb{R}^n \times \mathbb{R}^n$ with marginals $\mu$ and $\nu$. Now, we define the \emph{intrinsic spread} of $\mathcal{T}$,
\begin{equation}\label{eq:spread}
    V^* \;=\; \sup_{\mathbf{x}}\;\inf_{\mathbf{y}\in\mathbb{R}^n}\;
    W_2\!\left(\mathcal{T}(\cdot\mid\mathbf{x}),\;\delta_\mathbf{y}\right),
\end{equation}
as the best-possible $W_2$ distance from $\mathcal{T}(\cdot|\mathbf{x})$ to any point mass, where $\delta_\mathbf{y}$ denotes the Dirac measure centered at $\mathbf{y}$. Because the 2-Wasserstein distance to a Dirac mass is minimized exactly at the distribution's mean \cite{molinaro2024generative}, the inner infimum evaluates to the square root of the total variance:
\begin{equation*}
    \inf_{\mathbf{y}\in\mathbb{R}^n} W_2^2\!\left(\mathcal{T}(\cdot\mid\mathbf{x}),\;\delta_\mathbf{y}\right) \;=\; \mathbb{E}_{\mathbf{z}\sim\mathcal{T}(\cdot\mid\mathbf{x})}\!\left[\|\mathbf{z} - \mathbb{E}[\mathbf{z}]\|^2\right] \;=\; \mathrm{Tr}(\mathrm{Var}(\mathcal{T}(\cdot\mid\mathbf{x}))).
\end{equation*}
This makes $V^*$ a fundamental property of the true dynamics: $V^*>0$ for any non-degenerate kernel (where the variance is strictly positive for at least one state $\mathbf{x}$) and is naturally large when $\mathcal{T}$ is multimodal.

Now, we reintroduce the surrogate models used to approximate $\mathcal{T}$. Let $S_\theta : \mathbb{R}^n \to \mathbb{R}^n$ denote a deterministic surrogate mapping and $p_\theta(\cdot \mid \mathbf{x}) \in \mathcal{P}_2(\mathbb{R}^n)$ denote a probabilistic surrogate transition kernel, both parameterized by $\theta$. We compare two surrogate classes where $\mathbf{x}'$ is a free dummy variable.
\begin{itemize}
    \item \textbf{Deterministic}: $\hat{\mathbf{x}}_t = S_\theta(\hat{\mathbf{x}}_{t-1})$,
    with Lipschitz constant
    $L_{det} = \sup_{\mathbf{x}\neq\mathbf{x}'}
    \|S_\theta(\mathbf{x})-S_\theta(\mathbf{x}')\|/\|\mathbf{x}-\mathbf{x}'\|$.

    \item \textbf{Probabilistic}: $\hat{\mathbf{x}}_t \sim p_\theta(\cdot\mid
    \hat{\mathbf{x}}_{t-1})$, with Wasserstein-Lipschitz constant
    $L_{prob} = \sup_{\mathbf{x}\neq\mathbf{x}'}
    W_2(p_\theta(\cdot|\mathbf{x}), p_\theta(\cdot|\mathbf{x}'))/
    \|\mathbf{x}{-}\mathbf{x}'\|$.
\end{itemize}
Both surrogates start from the same known state $\mathbf{x}_0$, i.e.\ $\hat{\mu}_0=\mu_0=\delta_{\mathbf{x}_0}$, so $\mathcal{E}_0 = 0$. Thus $\mu_t$ and $\hat{\mu}_t$ are the true and surrogate forecast distributions at step $t$ given $\mathbf{x}_0$.

\subsection*{One-Step Error}

The one-step mismatch of each surrogate is the $W_2$ distance between the true kernel and the surrogate kernel, evaluated at the \emph{same} conditioning state:
\begin{align}
    \epsilon_{det}  &\;=\; \sup_{\mathbf{x}}\;
        W_2\!\left(\mathcal{T}(\cdot\mid\mathbf{x}),\;
        \delta_{S_\theta(\mathbf{x})}\right), \label{eq:eps_det}\\
    \epsilon_{prob} &\;=\; \sup_{\mathbf{x}}\;
        W_2\!\left(\mathcal{T}(\cdot\mid\mathbf{x}),\;
        p_\theta(\cdot\mid\mathbf{x})\right). \label{eq:eps_prob}
\end{align}
Here, the deterministic prediction $S_\theta(\mathbf{x})$ is lifted to a probability measure via the Dirac delta $\delta_{S_\theta(\mathbf{x})}$ to compute the Wasserstein distance \cite{molinaro2024generative}, whereas $p_\theta(\cdot\mid\mathbf{x})$ is inherently a valid distribution. By the definition of $V^*$ and as shown above, the deterministic mismatch satisfies $\epsilon_{det} \geq V^* > 0$ for \emph{any} $S_\theta$. The probabilistic mismatch satisfies $\epsilon_{prob} \to 0$ as $p_\theta \to \mathcal{T}$ (in $W_2$, uniformly in $\mathbf{x}$): that is, it can be made arbitrarily small given sufficient model capacity and optimal training \cite{molinaro2024generative}.

\subsection*{Autoregressive Error Bound}

\begin{proposition}\label{prop:main}
Assume that $\mu_0$ and the kernels $\mathcal{T}(\cdot\mid\mathbf{x})$ and $p_\theta(\cdot\mid\mathbf{x})$ for every $\mathbf{x}$ lie in $\mathcal{P}_2(\mathbb{R}^n)$. For a surrogate model with a finite one-step error $\epsilon \in \{\epsilon_{det}, \epsilon_{prob}\}$ and a Wasserstein-Lipschitz constant $1<L<\infty$, where $L \in \{L_{det}, L_{prob}\}$, all $\mu_t$ and $\hat\mu_t$ lie in $\mathcal{P}_2(\mathbb{R}^n)$ and the prediction error satisfies
\begin{equation}\label{eq:bound}
    \mathcal{E}_t \;\leq\;
    \frac{e^{\Lambda T} - 1}{\Lambda}\;\bar\epsilon,
    \qquad \Lambda = \frac{\ln L}{\Delta t},
    \quad \bar\epsilon = \frac{\epsilon}{\Delta t},
    \quad T = t\Delta t.
\end{equation}
\end{proposition}

\begin{proof}
For every $\mathbf{x}$, the triangle inequality gives
$W_2(p_\theta(\cdot\mid\mathbf{x}),\delta_{\mathbf{0}})\le L\|\mathbf{x}\|+C_0$ and
$W_2(\mathcal{T}(\cdot\mid\mathbf{x}),\delta_{\mathbf{0}})\le \epsilon+L\|\mathbf{x}\|+C_0$, with
$C_0:=W_2(p_\theta(\cdot\mid\mathbf{0}),\delta_{\mathbf{0}})<\infty$. Integrating the squares against
$\hat\mu_{t-1}$ and $\mu_{t-1}$ shows by induction that $\hat\mu_t,\mu_t\in\mathcal{P}_2(\mathbb{R}^n)$ for all $t$.

All measures below lie in $\mathcal{P}_2(\mathbb{R}^n)$, so optimal $W_2$-couplings exist and
all second moments are finite; by measurable selection \cite{villani2009optimal}, the conditional optimal couplings used below can be chosen measurably in the conditioning variables, as the gluing step requires. Let $(\mathbf{x}_{t-1},\hat{\mathbf{x}}_{t-1})$ realize an optimal
$W_2$-coupling of $(\mu_{t-1},\hat\mu_{t-1})$, so $\mathbb{E}\,\|\mathbf{x}_{t-1}-\hat{\mathbf{x}}_{t-1}\|^2=\mathcal{E}_{t-1}^2$.
To bound $\mathcal{E}_t=W_2(\mu_t,\hat\mu_t)$ we construct a (generally suboptimal) coupling of
$(\mu_t,\hat\mu_t)$ through the intermediate measure $p_\theta(\cdot\mid\mathbf{x}_{t-1})$, i.e.\ the
surrogate kernel evaluated at the \emph{true} previous state. Conditionally on
$(\mathbf{x}_{t-1},\hat{\mathbf{x}}_{t-1})$, draw
\begin{itemize}
  \item $\mathbf{x}_t\sim\mathcal{T}(\cdot\mid\mathbf{x}_{t-1})$ and
        $\mathbf{x}_t'\sim p_\theta(\cdot\mid\mathbf{x}_{t-1})$ under their optimal $W_2$-coupling;
  \item $\hat{\mathbf{x}}_t\sim p_\theta(\cdot\mid\hat{\mathbf{x}}_{t-1})$, coupled optimally to $\mathbf{x}_t'$.
\end{itemize}
By the gluing lemma these conditional couplings are realized on a common probability
space, and marginally $\mathbf{x}_t\sim\mu_t$, $\hat{\mathbf{x}}_t\sim\hat\mu_t$. Since this is a valid
coupling, $\mathcal{E}_t\le\|\mathbf{x}_t-\hat{\mathbf{x}}_t\|_{L^2(\Omega)}$, and Minkowski's inequality in
$L^2(\Omega)$ (inserting $\mathbf{x}_t'$) gives
\[
  \mathcal{E}_t \;\le\;
  \underbrace{\bigl\|\mathbf{x}_t-\mathbf{x}_t'\bigr\|_{L^2}}_{\mathrm{(I)}}
  \;+\;
  \underbrace{\bigl\|\mathbf{x}_t'-\hat{\mathbf{x}}_t\bigr\|_{L^2}}_{\mathrm{(II)}}.
\]
For (I), by construction the conditional second moment is the squared one-step distance at
$\mathbf{x}_{t-1}$, hence
\[
  \mathrm{(I)}^2
  =\mathbb{E}\Bigl[W_2^2\bigl(\mathcal{T}(\cdot\mid\mathbf{x}_{t-1}),\,p_\theta(\cdot\mid\mathbf{x}_{t-1})\bigr)\Bigr]
  \;\le\;\sup_{\mathbf{x}}W_2^2\bigl(\mathcal{T}(\cdot\mid\mathbf{x}),\,p_\theta(\cdot\mid\mathbf{x})\bigr)
  \;=\;\epsilon^2,
\]
so $\mathrm{(I)}\le\epsilon$. For (II), the Wasserstein-Lipschitz property applied pointwise,
$W_2(p_\theta(\cdot\mid\mathbf{x}_{t-1}),p_\theta(\cdot\mid\hat{\mathbf{x}}_{t-1}))\le L\,\|\mathbf{x}_{t-1}-\hat{\mathbf{x}}_{t-1}\|$,
yields
\[
  \mathrm{(II)}^2
  =\mathbb{E}\Bigl[W_2^2\bigl(p_\theta(\cdot\mid\mathbf{x}_{t-1}),\,p_\theta(\cdot\mid\hat{\mathbf{x}}_{t-1})\bigr)\Bigr]
  \;\le\;L^2\,\mathbb{E}\,\|\mathbf{x}_{t-1}-\hat{\mathbf{x}}_{t-1}\|^2
  \;=\;L^2\,\mathcal{E}_{t-1}^2,
\]
so $\mathrm{(II)}\le L\,\mathcal{E}_{t-1}$. Combining gives the recursion
\begin{equation}\label{eq:recursion}
  \mathcal{E}_t\;\le\;\epsilon+L\,\mathcal{E}_{t-1},\qquad \mathcal{E}_0=0.
\end{equation}
The deterministic surrogate is the special case $p_\theta(\cdot\mid\mathbf{x})=\delta_{S_\theta(\mathbf{x})}$,
for which $W_2(\delta_{S_\theta(\mathbf{x})},\delta_{S_\theta(\mathbf{x}')})=\|S_\theta(\mathbf{x})-S_\theta(\mathbf{x}')\|$
recovers $L=L_{det}$ and (I) reduces to $\epsilon_{det}$; the argument is otherwise identical.
Unrolling \eqref{eq:recursion} over $t$ steps (similar to \citet{chakraborty2024note}) with
$\mathcal{E}_0=0$,
\begin{equation}
  \mathcal{E}_t\;\le\;\epsilon\sum_{k=0}^{t-1}L^k\;=\;\epsilon\,\frac{L^t-1}{L-1}\qquad(L\neq1).
\end{equation}
Writing $L^t=e^{\Lambda T}$ and using $L-1\ge\ln L=\Lambda\Delta t>0$ (valid for $L>1$; equivalently $e^x\ge1+x$
with $x=\ln L>0$), which only loosens the bound since $e^{\Lambda T}-1\ge0$, yields
Equation~\eqref{eq:bound}. For chaotic dynamics one typically has $L>1$
\cite{eckmann1985ergodic,abraham2004lyapunov}; we stress, however, that $L$ is a \emph{global}
(worst-case) Wasserstein-Lipschitz constant of the surrogate, not a leading Lyapunov exponent,
so $L>1$ is a modeling expectation rather than an identity. For $L\le1$ the recursion~\eqref{eq:recursion} still gives $\mathcal{E}_t\le t\epsilon=T\bar\epsilon$ for $L=1$ (the $\Lambda\to0$ limit of~\eqref{eq:bound}) and $\mathcal{E}_t\le\epsilon/(1-L)$ for $L<1$; only the exponential form~\eqref{eq:bound} requires $L>1$.
\end{proof}

Applying Proposition~\ref{prop:main} to each surrogate gives the two explicit bounds:
\begin{equation}\label{eq:det_final}
    \mathcal{E}_t^{det} \;\leq\;
    \frac{e^{\Lambda_{det} T}-1}{\Lambda_{det}}\;\bar\epsilon_{det},
    \qquad \bar\epsilon_{det} \;\geq\; \bar{V}^* \; = \frac{V^*}{\Delta t} \;>\; 0,
\end{equation}
\begin{equation}\label{eq:prob_final}
    \mathcal{E}_t^{prob} \;\leq\;
    \frac{e^{\Lambda_{prob} T}-1}{\Lambda_{prob}}\;\bar\epsilon_{prob},
    \qquad \bar\epsilon_{prob} \;\to\; 0
    \text{ as } p_\theta \to \mathcal{T}.
\end{equation}

\textbf{Remark:} The two bounds differ in both factors. The per-step error $\bar\epsilon_{det}$
is bounded away from zero by $\bar V^*$ for any deterministic model, whereas $\bar\epsilon_{prob}$
is fully reducible. The growth rates $\Lambda_{det}$ and $\Lambda_{prob}$ are set by the global
Wasserstein-Lipschitz constants of the two surrogates (analogous to the Jacobian of the hybrid
solver in \citet{chakraborty2024note}). Under the assumption that $\Lambda_{det}$ and $\Lambda_{prob}$ can not be determined and controlled, the irreducible per-step floor $\bar V^*$ becomes the determining factor for reducing the error. This makes the probabilistic surrogate preferable and motivates further
architectural research toward minimizing the reducible term $\bar\epsilon_{prob}$.

\subsection{Diffusion Posterior Sampling with Sparse Sensors}\label{sec:sensor_theory}
We show that for an arbitrary prior $p(\mathbf{x}_0)$ with $\mathbb{E}[\|\mathbf{x}_0\|^2]<\infty$, the MMSE
reconstruction error equals the expected posterior uncertainty; that adding a sensor never
increases this error, and strictly decreases it whenever the new observation is informative;
and that placing the next sensor at the node of highest expected posterior variance maximizes a
guaranteed \emph{lower bound} on the per-step error reduction. We then connect these idealized
criteria to the heuristic sensor placement strategies of Section \ref{sec:sensor_placement}.

\paragraph{Setup.}
Let the ground-truth state $\mathbf{x}_0\in\mathbb{R}^{N}$ be a random vector with
$\mathbb{E}[\|\mathbf{x}_0\|^2]<\infty$ on a mesh of $N$ nodes (considering a single variable for
simplicity). A configuration of $s\leq N$ sensors (here $s$ counts the sensors placed so far, not the fixed budget of Section~\ref{sec:sensor_placement}) is encoded by a linear measurement operator
$M:\mathbb{R}^{N}\to\mathbb{R}^{s}$, selecting rows indexed by $\mathcal{S}_s\subset\{1,\dots,N\}$.
Observations follow
\begin{equation}
    \mathbf{y}_s = M\mathbf{x}_0 + \boldsymbol{\eta}, \qquad
    \boldsymbol{\eta}\sim\mathcal{N}(\mathbf{0},\sigma_\eta^2\mathbf{I}_s),\quad \sigma_\eta^2>0.
\end{equation}
Here $\boldsymbol{\eta}$ is independent of $\mathbf{x}_0$, $\sigma_\eta^2$ plays the role of $\Sigma_y=\sigma_{\mathrm{obs}}^2$ in Equation~\eqref{eq:likelihood_score_edm}, and each added sensor extends $\mathbf{y}_s$ by one entry with its own independent noise.
Let $\hat{\mathbf{x}}_s=\mathbb{E}[\mathbf{x}_0\mid\mathbf{y}_s]$ be the MMSE estimate and
\begin{equation}
    \mathcal{E}(s)\;=\;\mathbb{E}\bigl[\|\mathbf{x}_0-\hat{\mathbf{x}}_s\|^2\bigr]
\label{eq:mse_general}
\end{equation}
the mean-squared reconstruction error. The finite-second-moment assumption guarantees that all
conditional means, covariances, and traces below are finite and well defined. Throughout this
subsection $\mathcal{E}(\cdot)$ denotes the reconstruction MSE as a function of the sensor
count~$s$; it is distinct from the forecast error $\mathcal{E}_t$ of
Section~\ref{sec:error_analysis}, which is indexed by the rollout step~$t$. The results below concern the exact posterior mean under the true prior $p(\mathbf{x}_0)$. Diffusion posterior sampling produces approximate posterior samples, so they apply to its ensemble mean only to the extent that the learned prior and the sampler are accurate.

\paragraph{MSE equals expected posterior uncertainty.}
Since $\hat{\mathbf{x}}_s=\mathbb{E}[\mathbf{x}_0\mid\mathbf{y}_s]$, the residual $\mathbf{x}_0-\hat{\mathbf{x}}_s$ is the
posterior fluctuation around the conditional mean. Writing
$\|\mathbf{v}\|^2=\operatorname{Tr}(\mathbf{v}\mathbf{v}^\top)$ and interchanging trace and expectation,
then applying the tower property
$\mathbb{E}[(\mathbf{x}_0-\hat{\mathbf{x}}_s)(\mathbf{x}_0-\hat{\mathbf{x}}_s)^\top]=\mathbb{E}[\operatorname{Var}(\mathbf{x}_0\mid\mathbf{y}_s)]$,
\begin{equation}
    \mathcal{E}(s)
    \;=\;\operatorname{Tr}\!\Bigl(\mathbb{E}\bigl[\operatorname{Var}(\mathbf{x}_0\mid\mathbf{y}_s)\bigr]\Bigr)
    \;=\;\mathbb{E}\!\Bigl[\operatorname{Tr}\!\bigl(\operatorname{Var}(\mathbf{x}_0\mid\mathbf{y}_s)\bigr)\Bigr].
\label{eq:mse_trace}
\end{equation}
The MSE is therefore the \emph{expected total posterior variance}, which coincides with the
A-optimality criterion of classical optimal experimental design
\citep{chaloner1995bayesian,manohar2018data,ma2025physense}.

\begin{proposition}[Monotone reduction with sensor count]
\label{prop:monotone}
Let $\mathcal{S}_1\subset\mathcal{S}_2\subset\cdots\subset\{1,\ldots,N\}$ be a nested sequence of
index sets, where $M_{s+1}$ extends $M_s$ by adding a single sensor at node $i$. Then, at stage
$s$,
\begin{equation}
    \Delta_{s,i}
    \;:=\;\mathcal{E}(s)-\mathcal{E}(s+1)
    \;=\;\mathbb{E}\!\Bigl[\operatorname{Tr}\!\bigl(\operatorname{Var}(\mathbf{x}_0\mid\mathbf{y}_s)
        -\operatorname{Var}(\mathbf{x}_0\mid\mathbf{y}_{s+1})\bigr)\Bigr]
    \;=\;\mathbb{E}\bigl[\|\hat{\mathbf{x}}_{s+1}-\hat{\mathbf{x}}_s\|^2\bigr]
    \;\geq\;0,
\label{eq:reduction_trace}
\end{equation}
so $\mathcal{E}(s+1)\leq\mathcal{E}(s)$, with equality iff $\hat{\mathbf{x}}_{s+1}=\hat{\mathbf{x}}_s$ almost surely.
\end{proposition}

\begin{proof}
Since $\mathbf{y}_s$ is a sub-vector of $\mathbf{y}_{s+1}$, $\mathcal{F}(\mathbf{y}_s)\subseteq\mathcal{F}(\mathbf{y}_{s+1})$ for the generated $\sigma$-algebras, and
the conditional matrix law of total variance applied to $(\mathbf{x}_0,\mathbf{y}_s,\mathbf{y}_{s+1})$ gives
\begin{equation}
    \operatorname{Var}(\mathbf{x}_0\mid\mathbf{y}_s)
    =\mathbb{E}\bigl[\operatorname{Var}(\mathbf{x}_0\mid\mathbf{y}_{s+1})\mid\mathbf{y}_s\bigr]
    +\operatorname{Var}(\hat{\mathbf{x}}_{s+1}\mid\mathbf{y}_s),
\label{eq:matrix_lotv}
\end{equation}
with $\hat{\mathbf{x}}_{s+1}=\mathbb{E}[\mathbf{x}_0\mid\mathbf{y}_{s+1}]$. Both right-hand terms are positive
semidefinite. Taking trace and unconditional expectation and applying
Equation~\eqref{eq:mse_trace} at stages $s$ and $s+1$,
\[
    \mathcal{E}(s)=\mathcal{E}(s+1)+\mathbb{E}\bigl[\operatorname{Tr}\operatorname{Var}(\hat{\mathbf{x}}_{s+1}\mid\mathbf{y}_s)\bigr].
\]
By the tower property
$\mathbb{E}[\hat{\mathbf{x}}_{s+1}\mid\mathbf{y}_s]=\mathbb{E}[\mathbb{E}[\mathbf{x}_0\mid\mathbf{y}_{s+1}]\mid\mathbf{y}_s]=\hat{\mathbf{x}}_s$,
so $\mathbb{E}[\operatorname{Tr}\operatorname{Var}(\hat{\mathbf{x}}_{s+1}\mid\mathbf{y}_s)]=\mathbb{E}[\|\hat{\mathbf{x}}_{s+1}-\hat{\mathbf{x}}_s\|^2]$,
which is Equation~\eqref{eq:reduction_trace}; it vanishes iff $\hat{\mathbf{x}}_{s+1}=\hat{\mathbf{x}}_s$ a.s.
\end{proof}

\paragraph{Justification for uncertainty-guided placement.}
Write $\varsigma_i^2:=\operatorname{Var}(x_{0,i})$ for the prior (no-sensor) marginal variance at node~$i$
(for example, estimated directly from the unconditional diffusion ensemble), and, at stage~$s$,
\[
  V_{s,i}:=\operatorname{Var}(x_{0,i}\mid\mathbf{y}_s),\qquad
  v_{s,i}^2:=\mathbb{E}\!\bigl[V_{s,i}\bigr]
\]
for the random posterior variance and its expectation. We show that ranking nodes by $v_{s,i}$ maximizes a lower bound on the error reduction at every stage, and that the prior spread $\varsigma_i$ coincides with $v_{s,i}$ for the first placement and is a controlled proxy for it thereafter.

\medskip\noindent\textit{One-step guarantee.}\;
Expanding the trace in Equation~\eqref{eq:reduction_trace}, the reduction from a sensor at
node~$i$ decomposes over coordinates,
\[
  \Delta_{s,i}
  \;=\;\sum_{j=1}^{N}\Bigl(v_{s,j}^2-\mathbb{E}\!\bigl[\operatorname{Var}(x_{0,j}\mid\mathbf{y}_{s+1})\bigr]\Bigr),
\]
and by the scalar law of total variance every summand is non-negative. Retaining only the $j=i$
term lower-bounds the sum,
\begin{equation}\label{eq:delta_lb}
  \Delta_{s,i}\;\geq\;v_{s,i}^2-\mathbb{E}\!\bigl[\operatorname{Var}(x_{0,i}\mid\mathbf{y}_{s+1})\bigr].
\end{equation}
Conditioning on $\mathbf{y}_s$, the new scalar measurement is $y_{s+1,i}=x_{0,i}+\eta_i$ with
$\eta_i\sim\mathcal{N}(0,\sigma_\eta^2)$ independent of $(x_{0,i},\mathbf{y}_s)$. The error of the conditional MMSE estimator
cannot exceed that of the conditional linear (Wiener) estimator \citep{kay1993fundamentals}, whose
error variance is $\sigma_\eta^2 V_{s,i}/(\sigma_\eta^2+V_{s,i})$:
\[
  \mathbb{E}\!\bigl[\operatorname{Var}(x_{0,i}\mid\mathbf{y}_{s+1})\mid\mathbf{y}_s\bigr]
  \;\leq\;\frac{\sigma_\eta^2 V_{s,i}}{\sigma_\eta^2+V_{s,i}}.
\]
Taking unconditional expectations and applying Jensen's inequality to the concave map
$k\mapsto\sigma_\eta^2 k/(\sigma_\eta^2+k)$ (second derivative $-2\sigma_\eta^4/(\sigma_\eta^2+k)^3<0$),
with $\mathbb{E}[V_{s,i}]=v_{s,i}^2$,
\[
  \mathbb{E}\!\bigl[\operatorname{Var}(x_{0,i}\mid\mathbf{y}_{s+1})\bigr]
  \;\leq\;\mathbb{E}\!\left[\frac{\sigma_\eta^2 V_{s,i}}{\sigma_\eta^2+V_{s,i}}\right]
  \;\leq\;\frac{\sigma_\eta^2 v_{s,i}^2}{\sigma_\eta^2+v_{s,i}^2}.
\]
Substituting into~\eqref{eq:delta_lb},
\begin{equation}\label{eq:delta_bound}
  \Delta_{s,i}\;\geq\;\frac{v_{s,i}^4}{\sigma_\eta^2+v_{s,i}^2}\;=:\;\underline{\Delta}(v_{s,i}).
\end{equation}
The lower bound $\underline{\Delta}(\cdot)$ is zero iff $v_{s,i}=0$ and strictly increasing for
$v_{s,i}>0$, so $i^\star\in\arg\max_{i\notin\mathcal{S}_s} v_{s,i}$ \emph{maximizes the guaranteed
minimum reduction} $\underline{\Delta}(v_{s,i})$.

\medskip\noindent\textit{Scope of the guarantee.}\;
Equation~\eqref{eq:delta_bound} bounds $\Delta_{s,i}$ by its node-$i$ contribution alone,
discarding the non-negative variance reductions $\sum_{j\neq i}$ induced at correlated nodes.
The rule $\arg\max_i v_{s,i}$ is therefore a \emph{robust (maximin)} criterion - it maximizes a
lower bound on the reduction - rather than a maximizer of the actual reduction $\Delta_{s,i}$ or
an A-optimal greedy step (which maximizes the full $\Delta_{s,i}$). In strongly correlated fields
such as turbulence, where a single sensor reduces variance across an entire coherent structure,
the discarded cross-node terms can be substantial and the bound is correspondingly loose; exact
greedy A-optimal selection would in general place sensors differently and at higher online cost
\cite{manohar2018data, clark2020multi, karnik2024constrained, chaloner1995bayesian}.

\medskip\noindent\textit{From the posterior criterion to the deployed heuristics.}\;
The criterion ranks nodes by the posterior quantity $v_{s,i}$, which our methods approximate
through explicit substitutions:
\begin{enumerate}
  \item \emph{At $s=0$ the proxy is exact.} With no sensors, $V_{0,i}=\varsigma_i^2$ is deterministic
        and $v_{0,i}=\varsigma_i$, so placing the first sensor at $\arg\max_i \varsigma_i$ maximizes the
        guaranteed minimum reduction $\underline{\Delta}(\varsigma_i)$ exactly.
  \item \emph{Prior-for-posterior ($v_{s,i}\!\to\!\varsigma_i$).} For $s\geq1$ the bound requires the
        posterior $v_{s,i}$, which the uncertainty-driven method replaces by the prior $\varsigma_i$.
        Recomputing $v_{s,i}$ exactly would require regenerating a conditioned ensemble after every
        placement, which is prohibitive for large sensor budgets.
  \item \emph{Previous-step spread for the forecast prior.} The sensors for step $t+1$ are chosen before its forecast exists, so the spread of the posterior-sampled ensemble after step $t$ stands in for the prior spread $\varsigma_i$ of step $t+1$.
  \item \emph{Mean over channels.} For the $C$-channel node state of Section~\ref{sec:sensor_placement} the bound~\eqref{eq:delta_bound} holds channel by channel, $\Delta_{s,i}\ge\sum_c\underline{\Delta}(v_{s,i,c})$, with $v_{s,i,c}$ the analogue of $v_{s,i}$ for channel $c$; the channel mean of Equation~\eqref{uq_sensor_placement} aggregates these terms heuristically.
  \item \emph{Learned error for prior ($\varsigma_i\!\to\!\hat{\mathcal{E}}_i$).} The predictive method
        replaces the ensemble-based $\varsigma_i$ by the output $\hat{\mathcal{E}}_i$ of the error network
        $P_\phi$, removing the need for any ensemble at placement time.
\end{enumerate}
Two properties control the prior-for-posterior substitution:
\begin{enumerate}
\item By the unconditional law of total variance,
  $\varsigma_i^2=v_{s,i}^2+\operatorname{Var}\!\bigl(\mathbb{E}[x_{0,i}\mid\mathbf{y}_s]\bigr)\geq v_{s,i}^2$, so
  $\varsigma_i\geq v_{s,i}$ at every stage. A small prior variance therefore certifies a small posterior
  variance (hence a small guaranteed $\underline{\Delta}$), and such nodes are de-prioritized consistently with the lower-bound criterion. However, a large $\varsigma_i$ does \emph{not} imply a large
  $v_{s,i}$, since the posterior variance of a node strongly correlated with $\mathbf{y}_s$ may already
  have collapsed. Ranking merely by $\varsigma_i$ can over-rank such a node.
\item \emph{Suppression distance.}\;
  The topology-aware selection (Section~\ref{sec:sensor_placement}) removes candidates within a
  gap~$g$ of existing sensors - the nodes whose posterior variance is typically most reduced by
  those sensors when statistical dependence decays with mesh distance. Among the surviving eligible nodes the prior-to-posterior gap
  $\varsigma_i^2-v_{s,i}^2=\operatorname{Var}(\hat{x}_{s,i})$ is small to the extent that distant nodes are
  weakly dependent on $\mathbf{y}_s$, in which regime $v_{s,i}\approx \varsigma_i$ and the prior ranking
  tracks the posterior ranking.
\end{enumerate}

\paragraph{Learned error predictor as a lightweight alternative.}
The per-node identity $\mathbb{E}\bigl[(x_{0,i}-\hat{x}_{s,i})^2\bigr]=v_{s,i}^2$ (the coordinatewise
form of Equation~\eqref{eq:mse_trace}, valid for the posterior-mean estimate $\hat{x}_{s,i}$) shows that any estimator of the per-node squared reconstruction error of that estimate is itself a proxy for the posterior variance governing placement. The error
network $P_\phi$ is trained against the standardized spatial squared error $\tilde{\mathcal{E}}_{gt}$ of
Equation~\eqref{eq:pred_network}, which is the reconstruction error of the 
($\sigma_{\max}$, no-sensor) diffusion estimate and is therefore a node-level proxy for the prior
$\varsigma_i^2$ rather than the posterior $v_{s,i}^2$. The predictive method thus realizes the final
substitution $\varsigma_i\to\hat{\mathcal{E}}_i$ in the chain above: it is compatible with the
suppression-distance strategy and provides a computationally inexpensive surrogate for
uncertainty-guided placement that requires no ensemble at inference time.

\section{Results}\label{sec:results}

We evaluate the framework on two cases that use different neural network architectures for the denoiser $D_\theta$: two-dimensional turbulence (Section~\ref{sec:hit}) on a periodic $256\times256$ grid, with a convolutional U-Net (SongUNet \cite{song2020score}, 100.8\,M parameters), and the flow over a backward-facing step (Section~\ref{sec:bfs}) on an unstructured finite-element mesh with $N=11{,}930$ nodes, with the multi-scale graph transformer of Section~\ref{sec:graph_diffusion_arch} (5.37\,M parameters). The EDM objective, the deterministic EDM sampler, posterior sampling and the greedy sensor placement with suppression are shared; the case-specific settings are given in \appref{sec:hyperparams}. The backward-facing step, with its unstructured mesh and separated flow, is the main case of this work to test different hypotheses.

Unless stated otherwise, each forecast trajectory is an ensemble of $E$ members started over several initial conditions in the unseen test dataset. We use the single-step model for the turbulence case and the multi-step model for the backward-facing step based on their respective performance(no major differences \appref{sec:multistep_app}). With sensors, every member assimilates noiseless observations at every forecast step by diffusion posterior sampling. The three placements of Section~\ref{sec:sensor_placement}, which choose new sensors at every step, are called random, predicted (from the learned error map) and adaptive std (from the ensemble spread). The error measure is the ensemble RMSE, the mean over members of the per-member RMSE over all grid points or nodes and both velocity components. Each run is summarized by its average RMSE reduction, the mean over the forecast steps of $100\,(1-\mathrm{RMSE}_t/\mathrm{RMSE}_t^{\mathrm{unc}})$, where $\mathrm{RMSE}_t$ and $\mathrm{RMSE}_t^{\mathrm{unc}}$ are the ensemble RMSEs at step $t$ with and without sensors from the same initial condition. It is computed for each initial condition and then averaged over them (mean $\pm$ one standard error).

\subsection{Case 1: Two-dimensional turbulence}\label{sec:hit}
\begin{figure}[!htp]
    \centering
    \includegraphics[width=\linewidth]{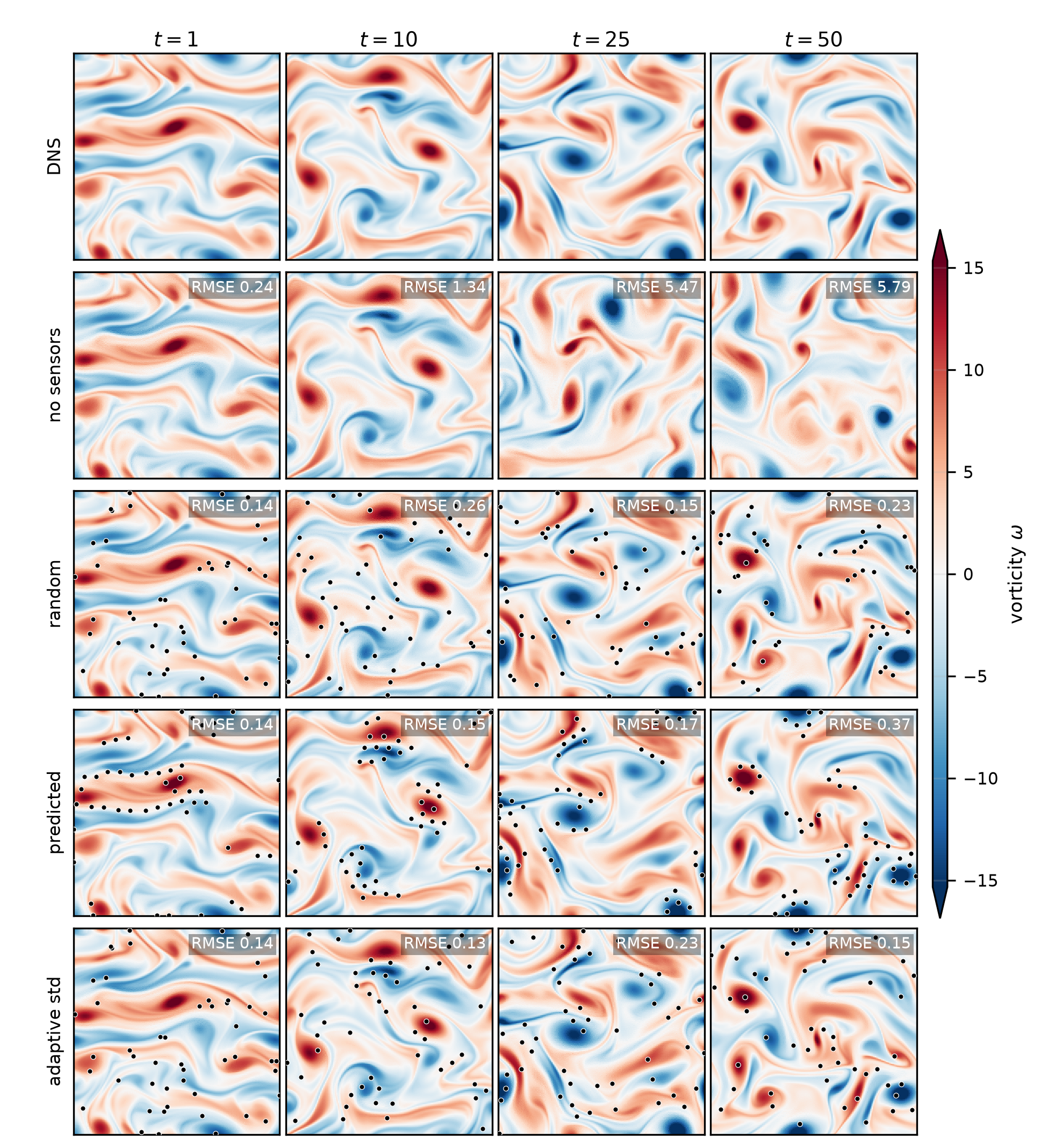}
    \caption{Vorticity $\omega$ of the DNS and of one ensemble member without sensors and with random, predicted and adaptive std sensors (black dots) under DPS, at forecast steps $t=1,10,25,50$ from the first test initial condition ($s=50$, $g=15$ grid points); each panel gives the vorticity RMSE of the member. At $t=1$, before an ensemble spread exists, the adaptive std row uses the layout of the random row.}
    \label{fig:turbgen_plots/different_sampling}
\end{figure}
\begin{figure}[!htp]
    \centering
    \includegraphics[width=0.62\linewidth]{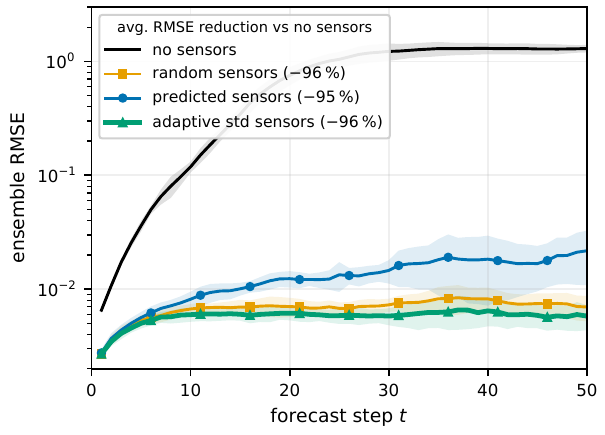}
    \caption{Ensemble RMSE (12 members, log scale) versus forecast step without sensors and with random, predicted and adaptive std sensors under DPS ($s=50$, $g=15$); mean over five test initial conditions, band $\pm1$ standard deviation across them. The legend gives the average RMSE reduction, printed with a minus sign.}
    \label{fig:turbgen_plots/different_sampling_error}
\end{figure}
The forced two-dimensional turbulence serves as a widely adopted benchmark in the forecasting of chaotic dynamical systems~\citep{stachenfeld2021learned,schiff2024dyslim,frerix2021variational}. We utilize a 2D homogeneous isotropic turbulent flow subjected to Kolmogorov forcing for building our diffusion forecast surrogate model. This system is dictated by the incompressible Navier-Stokes equations:

\begin{equation}
\begin{aligned}
\frac{\partial \mathbf{u}}{\partial t} + \nabla \cdot (\mathbf{u} \otimes \mathbf{u}) &= \frac{1}{Re} \nabla^2 \mathbf{u} - \nabla p + \mathbf{f}, \\
\nabla \cdot \mathbf{u} &= 0,
\end{aligned}
\label{eqn:NS_Equations}
\end{equation}

where $\mathbf{u} = (u, v)$ denotes the nondimensional velocity field, $p$ stands for the nondimensional pressure, and $Re$ represents the Reynolds number. The term $\mathbf{f}$ introduces the external forcing mechanism, formulated as:

\begin{equation}
\mathbf{f} = A \sin(ky)\hat{\mathbf{e}} - r\mathbf{u},
\end{equation}

where $\hat{\mathbf{e}}$ is the unit vector aligned with the $x$-axis. Following the configuration established by \citep{shankar2025differentiable}, our experimental setup employs an amplitude of $A = 1$, a wavenumber of $k = 4$, a linear drag coefficient of $r = 0.1$, and a Reynolds number of $Re = 1000$. The simulations are initialized using a stochastically generated divergence-free velocity field~\citep{kochkov2021machine}.

\begin{figure}[!htp]
    \centering
    \includegraphics[width=0.62\linewidth]{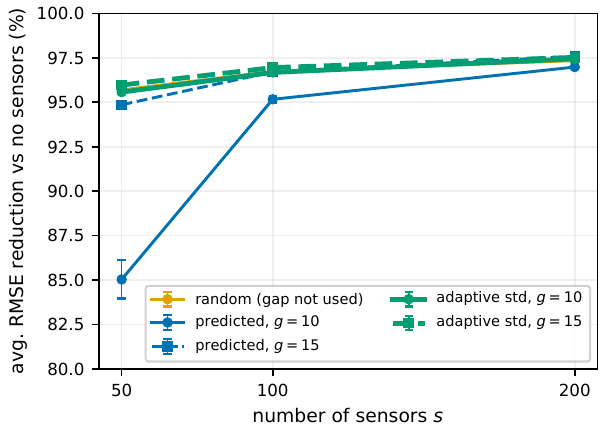}
    \caption{Average RMSE reduction versus the number of sensors $s$ for the three placement strategies under DPS (12 members, $g$ in grid points); mean over five test initial conditions, error bars one standard error.}
    \label{fig:turbgen_plots/points_gap_ablation}
\end{figure}

To establish our ground truth datasets, we rely on direct numerical simulations (DNS)~\citep{kochkov2021machine} applied to the governing equations. These simulations are conducted over a doubly periodic square spatial domain measuring $L_{\mathrm{box}} = 2\pi$. The spatial domain is discretized on a uniform $256 \times 256$ grid. The state therefore has $N=65{,}536$ grid points per velocity component, of which the sensor budgets $s=50$, 100 and 200 used below observe 0.08, 0.15 and 0.31\,\%. We extract temporal trajectory data only after the flow has fully transitioned into its chaotic regime. To guarantee meaningful variations between successive observations, flow snapshots are recorded at intervals of $\Delta t_{\mathrm{snap}} = 256\,\Delta t_{\mathrm{DNS}}$. They form 32 consecutive segments of one long trajectory (30 for training, one for validation and one for testing), and the forecasts of this section start from five snapshots of the test segment (Table~\ref{tab:combined_hyperparams}). For a more comprehensive breakdown of the dataset generation pipeline, readers are directed to~\citep{shankar2025differentiable}. The model and training hyperparameters are mentioned in \appref{sec:hyperparams}.

\begin{figure}[!htp]
    \centering
    \includegraphics[width=\linewidth]{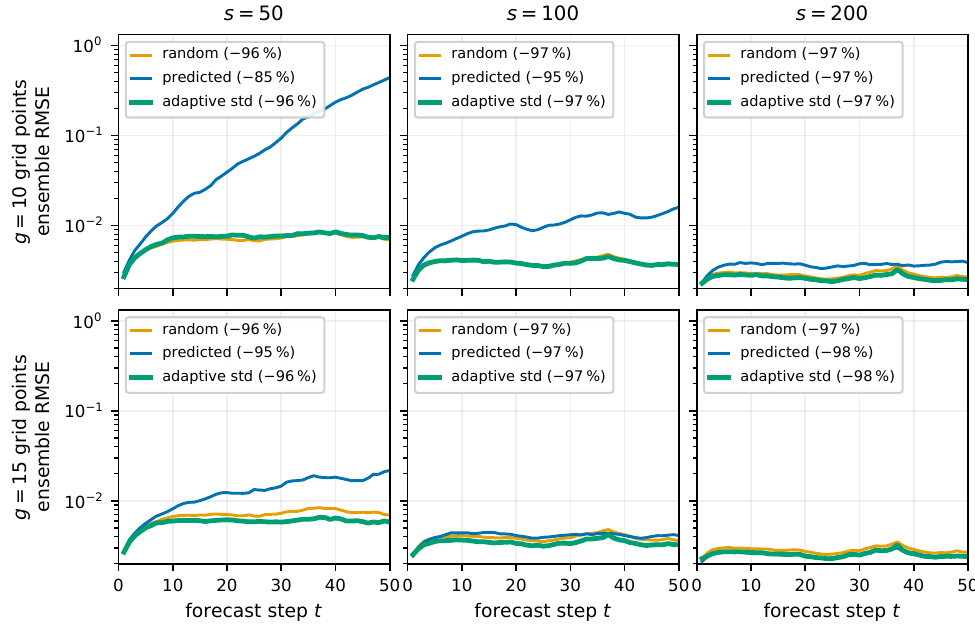}
    \caption{Ensemble RMSE (12 members, log scale) versus forecast step for each placement strategy under DPS, number of sensors $s$ (columns) and gap $g$ (rows; random sensors do not use $g$); mean over five test initial conditions. The legends give the average RMSE reduction, printed with a minus sign.}
    \label{fig:turbgen_plots/points_gap_mask_ablation}
\end{figure}

Figure~\ref{fig:turbgen_plots/different_sampling} shows that the predicted sensors gather around the coherent vortices as these move, whereas the adaptive std sensors favour them only slightly. Without sensors the forecast decorrelates from the DNS after about 25 steps (ensemble RMSE 0.12 at $t=10$ and 1.3 at $t=50$; Figure~\ref{fig:turbgen_plots/different_sampling_error}). With $s=50$ sensors and $g=15$, every strategy keeps the ensemble on the true trajectory for all 50 steps (average RMSE reduction 95-96\,\%), and more sensors lower the error further (97-98\,\% for $s=200$; Figure~\ref{fig:turbgen_plots/points_gap_ablation}). As expected, the adaptive std placement, which approximates the criterion of Equation~\eqref{eq:delta_bound}, is the most accurate overall: it is the best strategy or within 0.2 points of it in every configuration (Figure~\ref{fig:turbgen_plots/points_gap_ablation}; per-step curves and values in Figure~\ref{fig:turbgen_plots/points_gap_mask_ablation}). For $g=15$ and $s=50$ or 100 it is the best on each of the five initial conditions, but only by 0.2-0.3 points (96.0 against 95.7\,\% for random sensors at $s=50$). The predicted placement needs the larger gap: with $g=10$ and $s=50$. The other hyperparameters are in \appref{sec:hit_hyperparams}. We see on average, the RMSE reduction drops to 85\,\%, against 94.9\,\% with $g=15$. Since this is a homogeneous isotropic turbulence, there is no preferred direction of the flow, so all strategies show closely similar performance.

\begin{figure}[!htp]
  \centering
  \includegraphics[width=\linewidth]{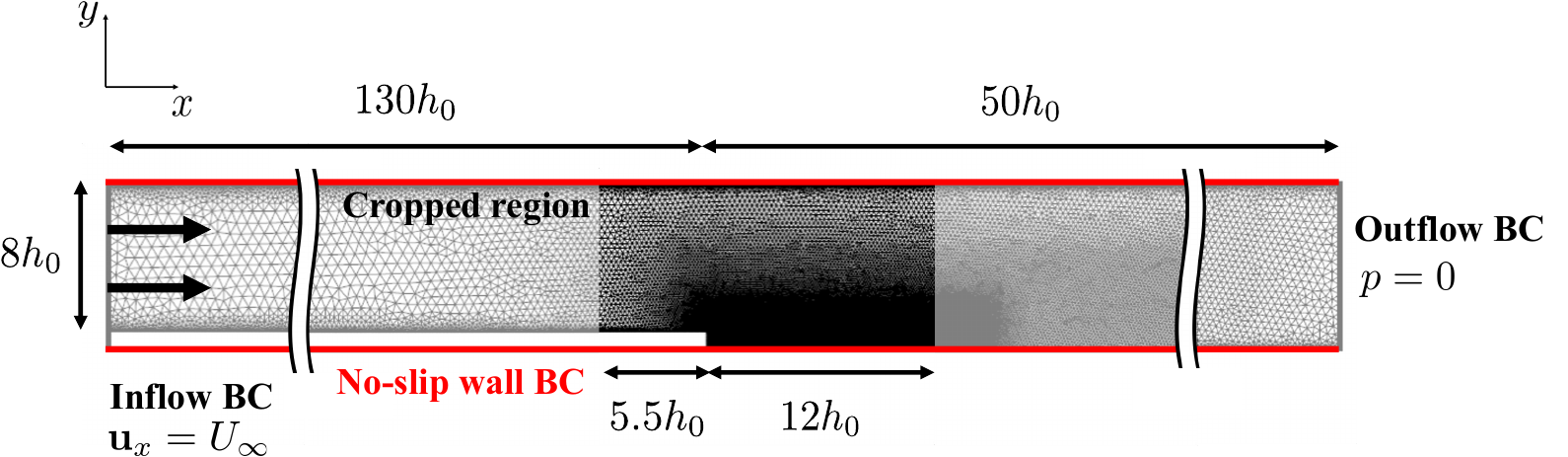}
  \caption{Visualization of the two-dimensional backward-facing step computational domain, including boundary conditions and the cropped region used for training; $h_0=h$ is the step height}
  \label{fig:domain_bfs}
\end{figure}

\subsection{\texorpdfstring{Case 2: Flow over a backward-facing step}{Case 2: Flow over a backward-facing step}}\label{sec:bfs}

Flow over a backward-facing step (BFS) is a widely studied benchmark problem in fluid dynamics, encompassing both experimental and computational approaches \cite{armaly1983experimental}. BFS flow is characterized by flow separation at the step and subsequent reattachment downstream, resulting in a recirculation region near the step and transient vortex shedding. These phenomena are prevalent in more complex, non-idealized flows and pose challenges in accurate modeling. Therefore, BFS serves as a valuable case study for evaluating modeling approaches. Similar to the forced two-dimensional turbulence case, BFS flow is governed by the incompressible Navier-Stokes equations as in Eq. \ref{eqn:NS_Equations}. A schematic of the computational domain is presented in Fig. \ref{fig:domain_bfs}.

Based on this configuration, a large eddy simulation (LES) under the Reynolds number, $Re = U_{\infty}h_{step}/\nu = 26,000$, is performed using the finite element method (FEM) to generate the ground truth dataset, and the Smagorinsky model \cite{smagorinsky1963general} with $C_s = 0.06$ is used for the subgrid-scale stress (SGS) model. For the FEM solver, FEniCSx \cite{baratta2023dolfinx} is used. The total grid for the two-dimensional BFS consists of 70,083 triangular Taylor-Hood elements with local refinement near the step and the walls. Even though the simulation is run with the full domain, to efficiently reduce the training time and GPU memory usage, only the flowfield in the cropped subdomain near the step anchor is used as the training dataset as in Fig. \ref{fig:domain_bfs}. The timestep for the simulation is set to $\Delta t = 1 \times 10^{-4}s$, which is sufficient to keep the CFL number less than 1. Of the 1{,}024 snapshots (one per time step), we discard the first 100 as initial transient and split the remaining 924 chronologically: the first 80\,\% for training and validation, the last 20\,\% for testing. The results in this section come from rollouts of 40 forecast steps, i.e.\ 4\,ms or about $10\,h/U_\infty$, where $U_\infty=31.7$\,m/s is the bulk velocity of the inlet channel and $h=12.7$\,mm the step height. The rollouts start from five initial conditions, one drawn at random from each fifth of the test snapshots that leave room for 40 steps (snapshots 866, 881, 902, 953 and 979). Due to the unstructured nature of this problem we utilize the graph neural network architecture mentioned in Section \ref{sec:graph_diffusion_arch}. Further details on the model, training and sampling hyperparameters are mentioned in \appref{sec:hyperparams}.

\begin{figure}[!htp]
    \centering
    \includegraphics[width=\linewidth]{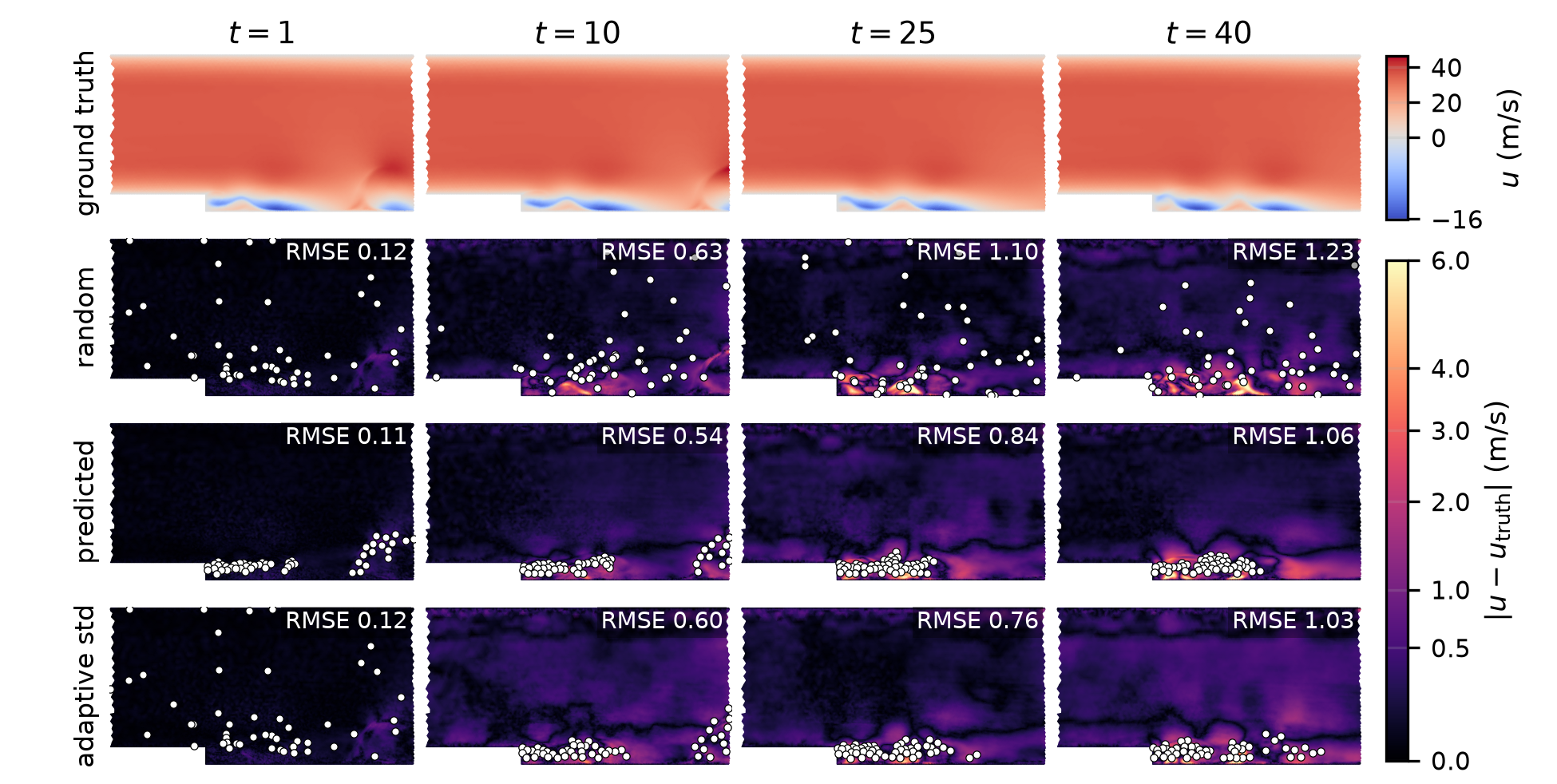}
    \caption{Ground-truth streamwise velocity $u$ (top row) and absolute error $|u-u_{\mathrm{truth}}|$ of one ensemble member for random, predicted and adaptive std sensors under DPS, at forecast steps $t=1,10,25,40$ from the initial condition at test snapshot 902 ($s=50$, $g=5$ hops). White dots mark the sensors used at that step; each error panel gives the RMSE of $u$ for that member (m/s). At $t=1$, before an ensemble spread exists, the adaptive std row uses the layout of the random row.}
    \label{fig:bfs_plots/different_sampling}
\end{figure}
\begin{figure}[!htp]
    \centering
    \includegraphics[width=0.62\linewidth]{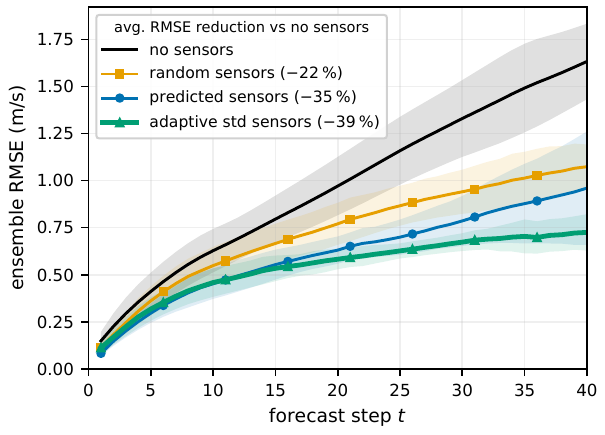}
    \caption{Ensemble RMSE (24 members, m/s) versus forecast step without sensors and with random, predicted and adaptive std sensors under DPS ($s=50$, $g=5$ hops); mean over the five initial conditions, band $\pm1$ standard deviation across them. The legend gives the average RMSE reduction over the 40 steps, printed with a minus sign.}
    \label{fig:bfs_plots/different_sampling_error}
\end{figure}

\label{sec:obs_model}At every forecast step the sensors report both components of the true velocity at their nodes, without noise (\appref{sec:bfs_noise} adds noise); the default $s=50$ sensors observe 0.42\,\% of the mesh nodes. All DPS runs use the default guidance, $\beta=1$ and $c_{\mathrm{clip}}=1$ (more in Table~\ref{tab:combined_hyperparams}). \appref{sec:bfs_guidance} reports a guidance-strength ablation from earlier runs on one initial condition. Figure \ref{fig:bfs_plots/different_sampling} shows that both the predicted and the adaptive std placement pick points near the step, which is the most turbulent part of the flow. Figure \ref{fig:bfs_plots/different_sampling_error} also shows that the adaptive std and predicted placements have lower errors than random sensor placement and no sensor placement. This is because of the presence of localized structures in the flow that are not captured in the random sensor placement. Figure~\ref{fig:bfs_plots/different_sampling} makes it clear that the member with random sensors has its largest errors in the shear layer behind the step, where the predicted and adaptive std sensor placement adapt to.

\begin{figure}[!htp]
    \centering
    \includegraphics[width=0.62\linewidth]{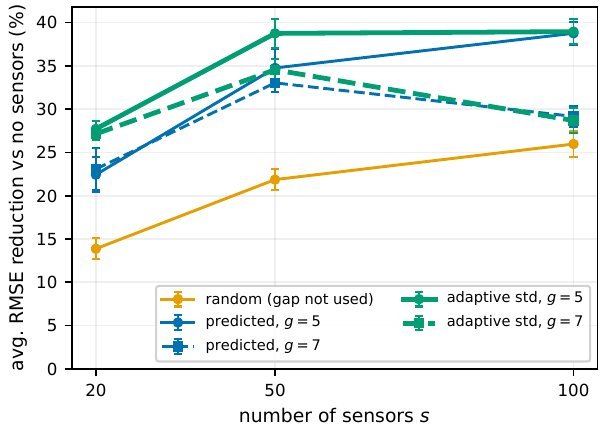}
    \caption{Average RMSE reduction versus the number of sensors $s$ for the three placement strategies under DPS (solid: $g=5$, dashed: $g=7$ hops; random sensors do not use $g$); mean over the five initial conditions, error bars one standard error.}
    \label{fig:bfs_plots/points_gap_ablation}
\end{figure}

\begin{figure}[!htp]
    \centering
    \includegraphics[width=\linewidth]{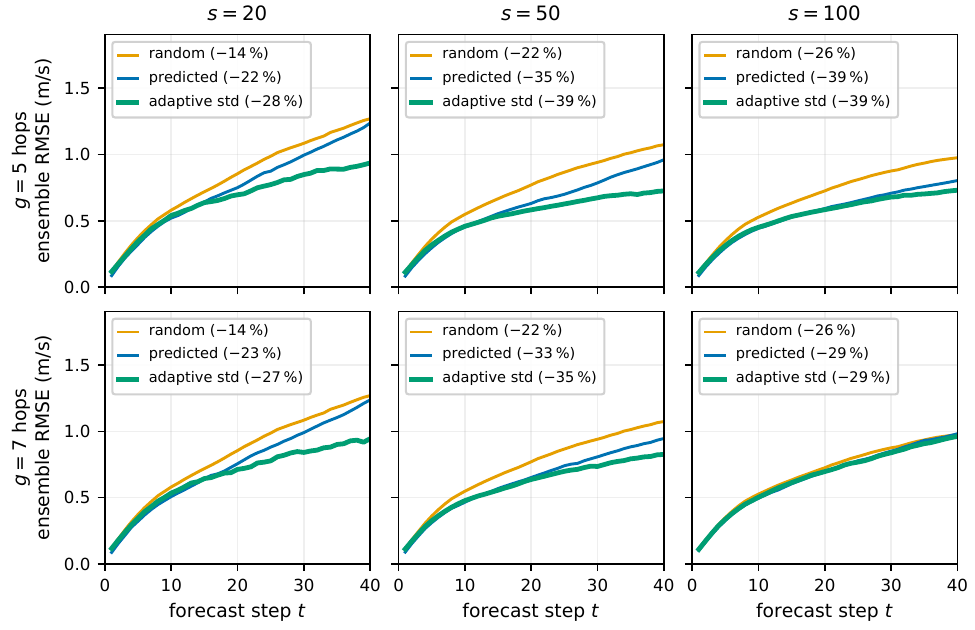}
    \caption{Ensemble RMSE (24 members, m/s) versus forecast step for each placement strategy under DPS, number of sensors $s$ (columns) and gap $g$ (rows; random sensors do not use $g$); mean over the five initial conditions. The legends give the average RMSE reduction, printed with a minus sign.}
    \label{fig:bfs_plots/points_gap_mask_ablation}
\end{figure}

Figures~\ref{fig:bfs_plots/points_gap_ablation} and \ref{fig:bfs_plots/points_gap_mask_ablation} vary the number of sensors $s$ and the gap $g$. Twenty adaptive std sensors reduce the error as much as 100 random sensors (average RMSE reduction 27.7 against 26.0\,\%). For $s=20$ and $50$ the adaptive std placement is the best strategy. For $s=100$ the two informed strategies, predicted and adaptive std, tie. For these two strategies $g=5$, is as good as or better than $g=7$ at every sensor budget. At $g=7$ all three strategies, random, predicted and adaptive std, are similar with $s=100$, plausibly because the larger gap pushes part of the budget out of the shear layer. All other results use $s=50$ and $g=5$ due to it being an intermediate configuration.

\subsubsection{Fixed versus adaptive sensors}\label{sec:bfs_fixed}
Re-selecting the sensors at every step from the current ensemble spread pays off, and the gain grows with the forecast horizon. Under DPS, we compare the adaptive std placement with four layouts that are chosen once and kept for the whole rollout (Figure~\ref{fig:bfs_fixed_layouts}). The uniform layout spreads the $s=50$ sensors evenly over the eligible interior nodes by farthest-point sampling. The fixed std layout ranks each node by the largest ensemble spread it reaches in a 20-step pre-rollout without sensors from the same initial condition and applies the greedy selection of Section~\ref{sec:sensor_placement}. Two established layouts, QR factorization with column pivoting (QR pivoting) \cite{manohar2018data} and the discrete empirical interpolation method (DEIM) \cite{chaturantabut2010nonlinear}, use the leading $r=50$ modes of a proper orthogonal decomposition (POD) of the standardized training snapshots and pivot on whole nodes, so that each sensor observes both velocity components. Both are reconstruction criteria for a fixed linear basis. An observability-Gramian criterion \cite{bopardikar2021randomized,brace2022sensor} would need a linearized dynamical model and is kept for future work.

Among the fixed layouts, the uniform one reaches an average RMSE reduction of only 8.2\,\%, QR pivoting and DEIM 28.5 and 28.6\,\% (more than random sensors) and the fixed std layout the most, 31.3\,\% (Table~\ref{tab:bfs_baselines}). As expected, the adaptive std placement (38.8\,\%) has better performance than the fixed std based placement layout(Figure~\ref{fig:bfs_baselines_dps_enkf}a). A fixed sensor placement strategy, although practical, cannot follow the shed vortices.

\subsubsection{Comparison with an ensemble Kalman filter}\label{sec:bfs_baselines}
To separate the contribution of DPS from that of the observations themselves, we also run the EnKF of Section~\ref{sec:enkf}, including its continuity projection (\appref{sec:bfs_continuity}), with the same sensor budget and placement strategies. Averaged over the 40 steps, DPS is more accurate than the EnKF with the adaptive std and every fixed layout, except for random sensors where the EnKF is more accurate (Table~\ref{tab:bfs_baselines}).

Secondly, we observe that DPS is more accurate than the EnKF early in the rollout, the EnKF is more accurate at later steps(Figure~\ref{fig:bfs_baselines_dps_enkf}a,b). During the first steps the EnKF's continuity projection over-smooths the state (\appref{sec:bfs_continuity}, Table~\ref{tab:bfs_metrics}), which reduces accuracy while the forecast error is still small. Once the forecast error has grown, the localized covariance update pays off, and at $t=40$ the EnKF has the lower error for every layout (0.59 versus 0.72\,m/s with adaptive std sensors). Its early error also makes the EnKF results vary more across initial conditions (Table~\ref{tab:bfs_baselines}).

\begin{figure}[!htp]
    \centering
    \includegraphics[width=\linewidth]{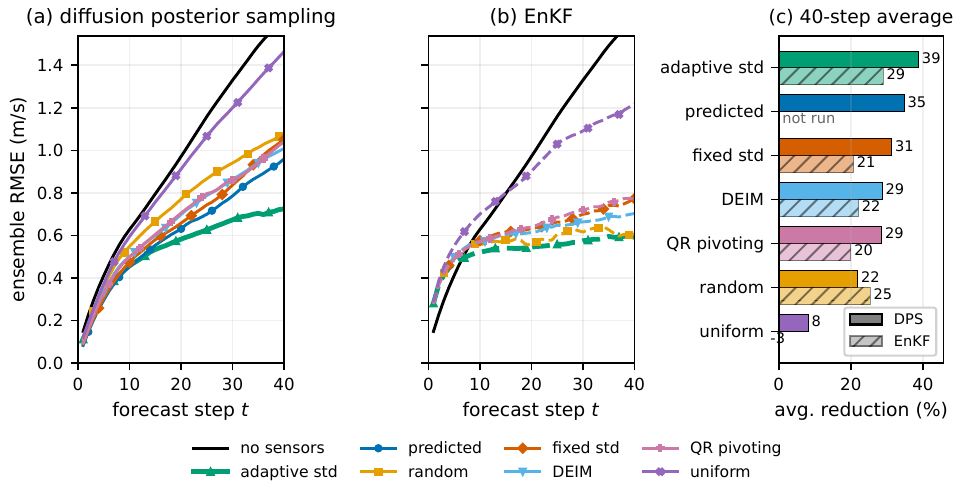}
    \caption{Placement baselines (the runs of Table~\ref{tab:bfs_baselines}). Ensemble RMSE (m/s) versus forecast step, mean over the five initial conditions, (a) with DPS and (b) with the EnKF (dashed; no predicted run). (c) Average RMSE reduction over the 40 steps (solid bars DPS, hatched bars EnKF).}
    \label{fig:bfs_baselines_dps_enkf}
\end{figure}
\begin{table}[!htp]
\centering
\caption{Average RMSE reduction (\%, mean $\pm$ standard error over the five initial conditions) and, in brackets, the mean ensemble RMSE at $t=40$ (m/s) of every sensor layout with DPS (default guidance) and with the EnKF; without sensors the RMSE at $t=40$ is 1.63\,m/s. All runs use $s=50$, $g=5$ hops and 24 members, with 40 forecast steps from each initial condition.}
\label{tab:bfs_baselines}
\small
\begin{tabular}{lcc}
\toprule
sensor layout & DPS & EnKF \\
\midrule
\multicolumn{3}{l}{\emph{moving}} \\
adaptive std (ensemble spread) & \textbf{38.8}\,$\pm$\,1.6 (0.72) & \textbf{28.9}\,$\pm$\,5.0 (0.59) \\
predicted (learned error map) & 34.8\,$\pm$\,2.2 (0.96) & not run \\
random (redrawn every step) & 21.9\,$\pm$\,1.3 (1.07) & 25.3\,$\pm$\,5.3 (0.61) \\
\midrule
\multicolumn{3}{l}{\emph{fixed}} \\
fixed std (pre-rollout spread) & 31.3\,$\pm$\,1.1 (1.05) & 20.7\,$\pm$\,4.6 (0.77) \\
DEIM (POD) & 28.6\,$\pm$\,0.9 (1.01) & 22.1\,$\pm$\,5.4 (0.70) \\
QR pivoting (POD) & 28.5\,$\pm$\,0.7 (1.04) & 20.0\,$\pm$\,5.4 (0.77) \\
uniform (farthest point) & 8.2\,$\pm$\,0.8 (1.46) & $-3.3$\,$\pm$\,4.3 (1.22) \\
\bottomrule
\end{tabular}
\end{table}

\begin{figure}[!htp]
    \centering
        \includegraphics[width=\linewidth]{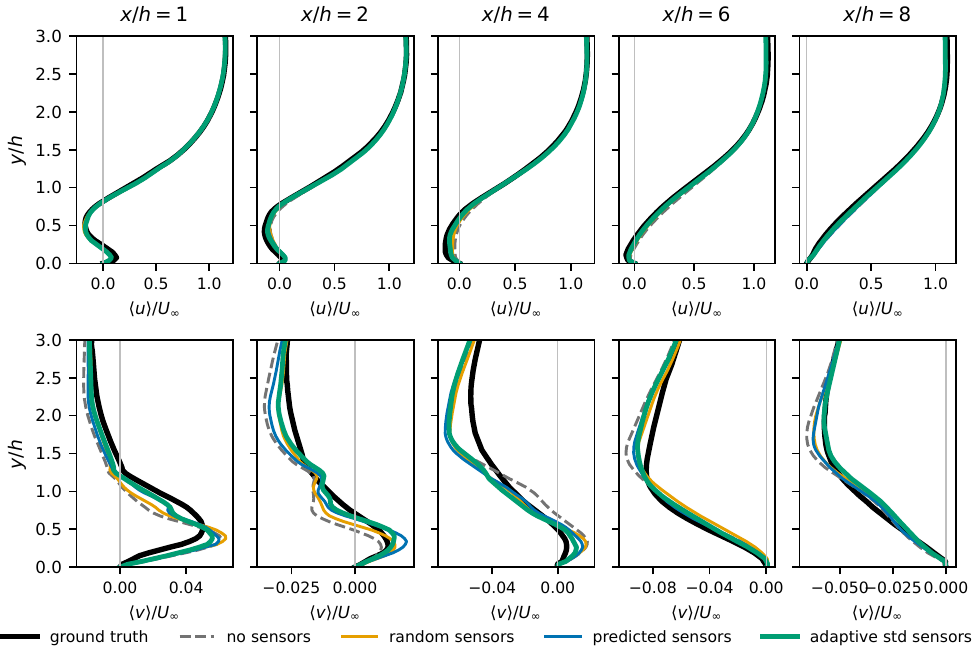}
        \caption{Time-averaged velocity $\langle u\rangle/U_\infty$ (top) and $\langle v\rangle/U_\infty$ (bottom) at $x/h=1,2,4,6,8$ for the ground truth, the rollout without sensors and random, predicted and adaptive std sensors with DPS ($s=50$, $g=5$ hops). Averages over 40 forecast steps, all 24 members and the five initial conditions (ground truth over the same steps); $U_\infty$ and $h$ as in Section~\ref{sec:bfs}.}
        \label{fig:bfs_plots/mean_profiles}
\end{figure}

\begin{figure}[!htp]
        \centering
        \includegraphics[width=\linewidth]{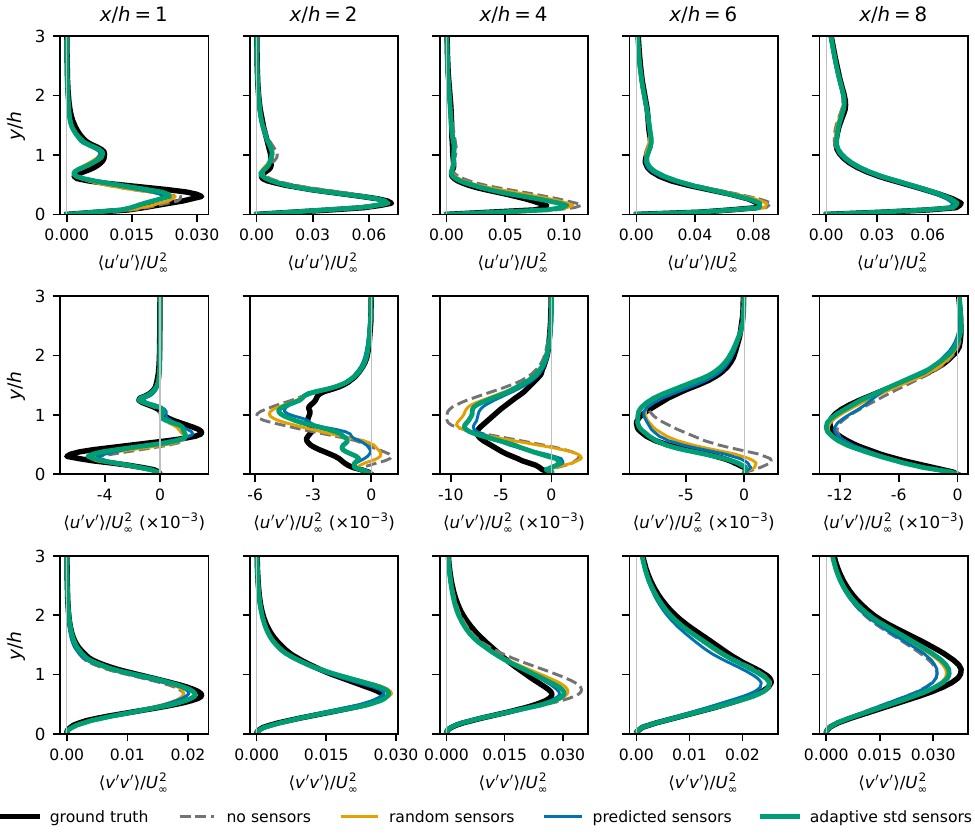}
        \caption{Time-averaged Reynolds stresses $\langle u'u'\rangle/U_\infty^2$ (top), $\langle u'v'\rangle/U_\infty^2$ (middle; axis in units of $10^{-3}$) and $\langle v'v'\rangle/U_\infty^2$ (bottom); stations, runs and averaging as in Figure~\ref{fig:bfs_plots/mean_profiles}.}
        \label{fig:bfs_plots/2nd_profiles}
\end{figure}

\subsubsection{Time-averaged statistics}\label{sec:bfs_profiles}
Even without sensors, the surrogate reproduces the mean streamwise velocity and the reattachment length (within 2 and 2.5\,\% of the ground truth, as do all DPS runs); the sensors mainly improve the small wall-normal mean velocity and the Reynolds stresses. Figures~\ref{fig:bfs_plots/mean_profiles} and \ref{fig:bfs_plots/2nd_profiles} show the profiles, and Table~\ref{tab:bfs_profiles_stats} in \appref{sec:boundary_band} lists the errors for eight runs. The wall-normal mean velocity is much smaller in magnitude, so its relative errors are larger (0.17 without sensors, 0.11-0.13 with DPS). With DPS, the adaptive std placement halves the error of the shear stress $\langle u'v'\rangle$ (relative $L_2$ error 0.205, against 0.418 without sensors). The EnKF matches $\langle u'v'\rangle$ best, but its $\langle u'u'\rangle$ errors are larger than with DPS and its reattachment length is 6-8\,\% too short (Figure~\ref{fig:bfs_profiles_methods}).

The accurate mean streamwise velocity without sensors is expected: time averages of a statistically stationary flow depend mainly on the learned distribution. The Reynolds stresses are more sensitive, probably because over 40 steps they are set by the few vortices that pass the stations. Without sensors the forecast drifts from the true trajectory, so these vortices, and hence its Reynolds stresses, can differ from the true ones even if the learned distribution is right. Accordingly, among the DPS runs the errors of $\langle u'u'\rangle$ and $\langle u'v'\rangle$ follow the RMSE ranking of Table~\ref{tab:bfs_baselines}: adaptive std, predicted, QR pivoting, random.

\subsubsection{Probabilistic metrics and the error bound}\label{sec:bfs_metrics}
A pointwise error such as the RMSE measures how well a run tracks one trajectory, not the quality of the learned distribution or the calibration of the ensemble. We therefore add four metrics: the anomaly correlation coefficient (ACC) of the ensemble mean with the truth; the continuous ranked probability score (CRPS) \cite{gneiting2007strictly}, which scores the whole ensemble rather than its mean; the spread-skill ratio, which is about one for a calibrated ensemble; and the energy spectrum of $u$ in the shear layer. Figure~\ref{fig:bfs_metrics} shows four representative runs; Table~\ref{tab:bfs_metrics} in \appref{sec:bfs_metrics_table} defines the metrics and lists eight runs.

All runs stay correlated with the truth, even without sensors: the ACC exceeds 0.6 at all 40 steps. The CRPS ranks the runs roughly as the ensemble RMSE at $t=40$ does (Table~\ref{tab:bfs_baselines}): 0.28-0.32\,m/s with the EnKF, 0.30-0.35\,m/s with DPS and 0.53\,m/s without sensors.

All ensembles are under-dispersive: the spread-skill ratio is at most 0.24, far below one. The members start from the identical, exact state, so their spread contains no initial-condition uncertainty; perturbed initial states or covariance inflation would widen it. The EnKF ensembles are the narrowest (0.05-0.06), because the analysis also removes spread at every step. In future work, we will investigate the methods that can improve the spread of the ensemble.

In the shear-layer spectrum, DPS adds small-scale energy: 2.6-2.9 times the true energy at high wavenumbers, against 1.6 without sensors. The EnKF is slightly too smooth (0.76-0.96 of the true energy there), as expected from its projection (Section~\ref{sec:bfs_baselines}).

In terms of comparison with the theoretical results of Section~\ref{sec:error_analysis}, the bound of
Proposition~\ref{prop:main} cannot be evaluated from its definition, because $\epsilon$ and $L$ are suprema over
the state space. We therefore fit its recursion, $\mathcal{E}_t\le L\,\mathcal{E}_{t-1}+\epsilon$, to the ensemble
RMSE of the pure surrogate (without sensors). For the BFS, the error grows from 0.15 to 1.63\,m/s over
40 steps, and the fit gives $L=0.97$-$1.02$. The amplification is close to one, so the error grows almost
linearly. For the homogeneous isotropic turbulence case, the error grows by a factor of 1.2-1.3 per step until
it saturates after about 25 steps. Computing the Lyapunov exponent is not feasible for these systems, especially
for the BFS case where we have limited data.

\begin{figure}[!htp]
    \centering
    \includegraphics[width=\linewidth]{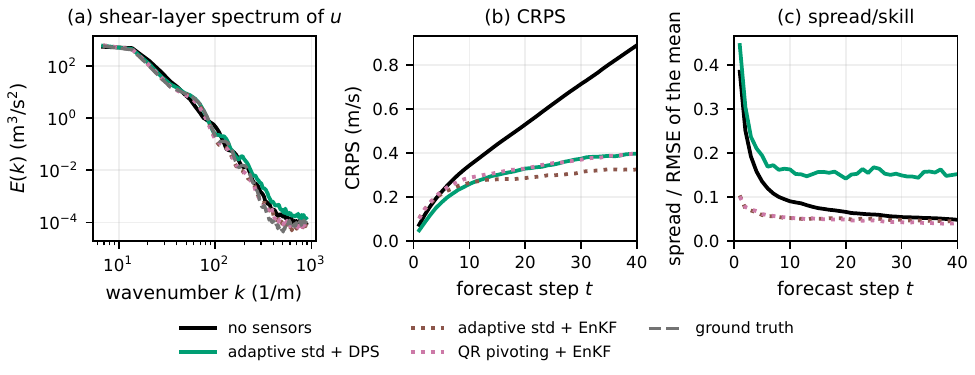}
    \caption{Probabilistic and physics-based metrics, mean over the five initial conditions ($s=50$, $g=5$ hops): (a) streamwise energy spectrum of $u$ in the shear layer over steps 20-40 (ground truth dashed), (b) CRPS and (c) spread-skill ratio versus forecast step, for the rollout without sensors, adaptive std sensors with DPS and with the EnKF, and QR pivoting with the EnKF. Table~\ref{tab:bfs_metrics} lists these and four more runs.}
    \label{fig:bfs_metrics}
\end{figure}
\subsection{Sampling Computational Costs}\label{sec:bfs_cost}
We compare the sampling costs for the different sensor placement techniques for homogeneous isotropic turbulence and the backward-facing step in Figures~\ref{fig:benchmark_sampling} and \ref{fig:bfs_benchmark_cost}. We used three Intel Data Center Max 1550 GPUs for the turbulence case (12 members, four per GPU) and 12 Intel Data Center Max 1550 GPUs for the backward-facing step case (24 members, two per GPU). Every method runs the full sampler: 99 denoiser evaluations (50 Heun steps) per member and forecast step. For the backward-facing step a forecast step without sensors takes 3.38\,s. DPS adds 50 backward passes through the denoiser for the likelihood gradient, which raises the sampling time to 4.80\,s (about $+42\,\%$). The localized update and the continuity projection of the EnKF raise it only to 3.4-3.6\,s with random sensors ($+1$-$7\,\%$). In the turbulence case DPS raises the sampling time from 9.1\,s to 24.7\,s. The greedy placement adds a time that grows linearly with the number of sensors: on the mesh 0.10, 0.24 and 0.48\,s for the adaptive std placement and 0.23, 0.52 and 0.98\,s for the predicted placement, which selects sensors for every member ($s=20$, 50 and 100); on the periodic grid 0.02-0.05\,s and 0.07-0.21\,s ($s=50$-200), below 1\,\% of a step. The times are measured on one initial condition per configuration, one run at a time on a dedicated node, without the first two steps (start-up). In both cases the predicted placement is not cheaper than the adaptive std placement as we keep the ensemble size fixed for both cases; we did not time it without an ensemble. We would like to mention that for the adaptive std based sensor placement, we will always need an ensemble to place the sensors, which would make the sampling cost substantially higher if compared against sampling without an ensemble generation. For a fair comparison we keep the ensemble size same in Figures~\ref{fig:benchmark_sampling} and \ref{fig:bfs_benchmark_cost}. The predicted sensor placement technique does not need any ensemble to place the sensors, which can make it ideal for a limited computational budget.

\begin{figure}[!htp]
    \centering
    \includegraphics[width=0.62\linewidth]{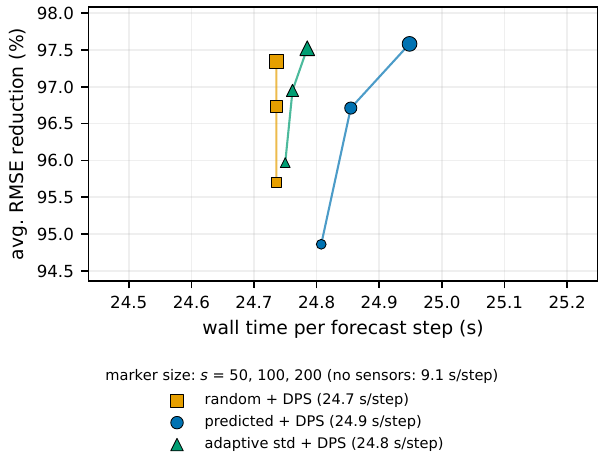}
    \caption{Cost-accuracy plane for the turbulence case (12 members, five test initial conditions): average RMSE reduction over 50 steps versus the wall time per forecast step (the DPS sampler time, the same for every strategy and budget, plus the measured placement time), for DPS with random, predicted and adaptive std sensors ($g=15$). The marker size grows with $s=50$, 100, 200. The legend gives each method's mean step time over the three budgets, and its title the step time without sensors.}
    \label{fig:benchmark_sampling}
\end{figure}
\begin{figure}[!htp]
    \centering
    \includegraphics[width=0.62\linewidth]{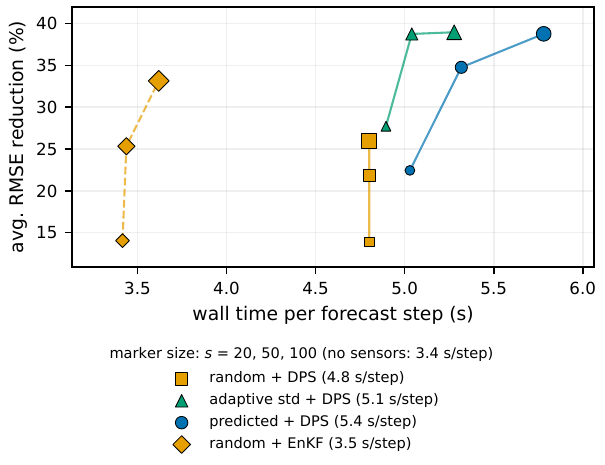}
    \caption{Cost-accuracy plane for the backward-facing step (24 members, five initial conditions): average RMSE reduction over 40 steps versus the wall time per forecast step (for DPS the sampler time, the same for every strategy and budget, plus the measured placement time; for the EnKF the measured step time), for DPS with random, predicted and adaptive std sensors ($g=5$) and for the EnKF with random sensors. The marker size grows with $s=20$, 50, 100. The legend gives each method's mean step time over the three budgets, and its title the step time without sensors.}
    \label{fig:bfs_benchmark_cost}
\end{figure}

\section{\texorpdfstring{Discussion and conclusions}{Discussion and conclusions}}

We have presented a novel framework for the probabilistic forecasting and data-assimilation of high-dimensional chaotic dynamical systems with exemplar applications such as turbulent flows. Our proposed approach uses deep learning diffusion models and provides stable long-term forecast horizons on unstructured grids, and unifies Bayesian data assimilation and computationally efficient optimal sensor placement. 

The error analysis of Section~\ref{sec:error_analysis} explains why a probabilistic surrogate suits chaotic dynamics: its per-step error from the true stochastic transition kernel, unlike that of a deterministic surrogate, is entirely reducible by training. On two-dimensional turbulence and the flow over a backward-facing step, we exploit the Bayesian interpretation of the diffusion model (as a training data-informed prior) to enable rapid data assimilation from sparse observations of the ground truth. Furthermore, estimates of uncertainty from the diffusion model also enable adaptive sensor placement at the cost of ensemble member generation. To mitigate this extra cost, we show that a metamodel trained to generate predictions for the parent diffusion model can also be used for efficient sensor placement decision making. In heterogeneous flows, such as flow over a backward-facing step, uncertainty-driven and predictive sensor placement strategies outperform random sampling by targeting specific regions that show higher degrees of turbulence and anisotropyInformed and adaptive sensor placement matters most when sensors are few: with 20 sensors the adaptive std placement doubles the average RMSE reduction of random sensors. Our topology-aware suppression reduces clustering where the mesh is refined or errors are high (Figure~\ref{fig:bfs_plots/different_sampling}). Our adaptive standard deviation based placement is theoretically justified in Section~\ref{sec:sensor_theory} and computationally established as the most accurate overall in the different scenarios that we have shown. We also compare these techniques against POD-based QR pivoting and DEIM based sensor placements. We also showcase experiments with the traditional EnKF-based data-assimilation technique.

Future direction of work will be focused on improving the calibration of the ensembles, especially targeting extreme event generation. Another direction can be to explore generating a window of several timesteps rather than purely autoregressive methods. While the current work focuses on two-dimensional experiments, the graph-based nature of the architecture makes it inherently scalable to three-dimensional geometries and complex multi-physics problems \cite{barwey2025mesh}. Future work can also extend this to more challenging cases and explore the integration of physics-informed constraints directly into the graph edge and node operations in the network to further enhance the conservation properties of the generated flow fields.

\section{Acknowledgments}

We acknowledge support from ARO cooperative agreement W911NF-25-2-0183 and an ARO ECP award from the Program `Modeling of Complex Systems' (PM - Dr. Rob Martin). This research used resources of the Argonne Leadership Computing Facility, which is a U.S. Department of Energy Office of Science User Facility operated under Contract DE-AC02-06CH11357. We acknowledge Shraddha Pathak, College of Information Sciences and Technology, Pennsylvania State University, for reviewing the error analysis in Section \ref{sec:error_analysis}. We also acknowledge computing resources from the Penn State Institute for Computational and Data Sciences.

\bibliographystyle{unsrtnat}
\bibliography{refs}

\appendix

\section{Training and Sampling Hyperparameters}\label{sec:hyperparams}

\subsection{Homogeneous Isotropic Turbulence}\label{sec:hit_hyperparams}

The data and the error predictor $P_\phi$ of Section~\ref{sec:hit} are listed in Table~\ref{tab:combined_hyperparams}. Both velocity components are standardized with the training statistics. The denoiser is the DDPM++-style SongUNet \cite{song2020score} with positional noise embedding (settings in Table~\ref{tab:combined_hyperparams}). $P_\phi$ uses BatchNorm, ReLU activations and dropout 0.1, and its weights are averaged (EMA). Its target here is the channel-averaged squared error of the frozen single-step denoiser on the noised next state, weighted by $1/\sigma^2$, averaged over three noise levels drawn from $[0.5,2]$, $[2,5]$ and $[5,10]$, and normalized to unit root mean square. Posterior sampling sums the squared residual over the observed entries of each member and clips the gradient element-wise at $c_{\mathrm{clip}}=0.1$ (Table~\ref{tab:combined_hyperparams}).

\subsection{\texorpdfstring{Backward-Facing Step}{Backward-Facing Step}}\label{sec:bfs_hyperparams}

We use the graph neural network diffusion model with EDM preconditioning mentioned in Section \ref{sec:graph_diffusion_arch}. The single-step and multi-step models (\appref{sec:multistep_app}) use this architecture, predict the next state itself rather than its increment, and are trained with the same settings except for the multi-step curriculum. Apart from these, we use similar training and sampling setup as for the homogeneous isotropic turbulence case. We use a spatial U-Net structure (voxel pooling, same lengthscales) but with $\texttt{hidden\_channels}=128$, no noise-level embedding, and GATv2 attention blocks \cite{brody2022attentive} (4 heads, dropout $0.1$) for the error prediction model; the diffusion backbone is frozen and the prediction model is trained with the loss of Equation~\eqref{eq:pred_network_label}. Table~\ref{tab:training_cost} gives the training cost of the two diffusion models and of $P_\phi$, which needs about 170 times fewer GPU-hours than the multi-step model. The other hyperparameters are mentioned in Table \ref{tab:combined_hyperparams}.

\begin{table}[!htp]
\centering
\caption{Consolidated training and sampling hyperparameters for Homogeneous Isotropic Turbulence (HIT) and Backward-Facing Step (BFS) experiments.\ kimg: thousands of training samples; $K$ ($n_{\mathrm{steps}}$): unroll steps. The multi-step HIT run continues the single-step run from 14{,}109\,kimg; its kimg count includes this shared part, but its training time does not.}
\label{tab:combined_hyperparams}
\small
\begin{tabular}{lcccc}
\toprule
 & \multicolumn{2}{c}{\textbf{Homogeneous Isotropic Turbulence}} & \multicolumn{2}{c}{\textbf{Backward-Facing Step (BFS)}} \\
\cmidrule(lr){2-3} \cmidrule(lr){4-5}
\textbf{Hyperparameter} & \textbf{Single} & \textbf{Multi} & \textbf{Single} & \textbf{Multi} \\
\midrule
\multicolumn{5}{l}{\textit{Data}} \\
\quad Train / val /test set & \multicolumn{2}{c}{61{,}410 / 2{,}047 / 2{,}047 pairs} & \multicolumn{2}{c}{80\,\% (train and val) / 20\,\% (test)} \\
\quad test initial conditions & \multicolumn{2}{c}{five} & \multicolumn{2}{c}{five} \\
\midrule
\multicolumn{5}{l}{\textit{Model Architecture}} \\
\quad Backbone type & \multicolumn{2}{c}{SongUNet} & \multicolumn{2}{c}{Custom GNN in Section \ref{sec:graph_diffusion_arch}} \\
\quad Base / Hidden channels & \multicolumn{2}{c}{32} & \multicolumn{2}{c}{256} \\
\quad Channel multipliers & \multicolumn{2}{c}{[1, 2, 4, 8, 16]} & \multicolumn{2}{c}{-} \\
\quad Attention resolutions & \multicolumn{2}{c}{[64, 64], [32, 32], [16, 16]} & \multicolumn{2}{c}{All levels (GATv2, 4 heads, attention dropout 0.1)} \\
\quad GATv2 bias / self-loops & \multicolumn{2}{c}{-} & \multicolumn{2}{c}{yes / no} \\
\quad Parameters & \multicolumn{2}{c}{100.8\,M} & \multicolumn{2}{c}{5.37\,M} \\
\quad Encoder/Decoder layers & \multicolumn{2}{c}{Standard} & \multicolumn{2}{c}{3 levels, 1 block each, 2 bottleneck blocks} \\
\quad Padding / Voxel scales & \multicolumn{2}{c}{Circular (periodic)} & \multicolumn{2}{c}{[0.015, 0.03, 0.06]} \\
\midrule
\multicolumn{5}{l}{\textit{Compute Details}} \\
\quad GPUs (Intel Max 1550) & 12 & 24 & \multicolumn{2}{c}{12} \\
\quad Training Time(hrs) & 50.2 & 18.4 & 58.5 & 183.8 \\
\midrule
\multicolumn{5}{l}{\textit{Optimizer \& Trainer}} \\
\quad Optimizer & AdamW & AdamW & AdamW & AdamW \\
\quad AdamW $\beta$ & \multicolumn{4}{c}{$(0.9, 0.95)$} \\
\quad Weight decay & \multicolumn{4}{c}{$10^{-4}$} \\
\quad Learning rate & $5\times10^{-4}$ & $5\times10^{-4}$ & $5\times10^{-4}$ & $5\times10^{-4}$ \\
\quad LR decay & \multicolumn{4}{c}{cosine, to 1\,\%} \\
\quad EMA half-life (kimg) & \multicolumn{4}{c}{50} \\
\quad Precision & \multicolumn{4}{c}{bfloat16 mixed} \\
\quad Unroll steps ($n_{\mathrm{steps}}$) & 1 & 4 & 1 & 4 \\
\quad Total training (kimg) & 30,000 & 19,604 & 25,000 & 25,000 \\
\quad LR ramp-up (kimg) & 800 & 800 & 1,000 & 1,000 \\
\quad Curriculum (kimg) & - & $K{=}2,3,4$ at 14/16/18k & - & $K{=}2,3,4$ at 10/15/20k \\
\quad Global batch size & 96 (8 per GPU) & 96 (4 per GPU) & 48 & 48 \\
\midrule
\multicolumn{5}{l}{\textit{EDM Preconditioning \& Sampling}} \\
\quad $\sigma_{\min}$ / $\sigma_{\max}$ & 0.002 / 80 & 0.002 / 80 & 0.002 / 80 & 0.002 / 80 \\
\quad $P_{\mathrm{\mu}}$ / $P_{\sigma}$ & $0$ / $2$ & $0$ / $2$ & $0$ / $2$ & $0$ / $2$ \\
\quad Sampler steps & 50 & 50 & 50 & 50 \\
\quad $\rho$ schedule & 7 & 7 & 7 & 7 \\
\quad $S_{\mathrm{churn}}$ / $S_{\mathrm{noise}}$ & 0 / 1 (det.\ Heun) & 0 / 1 (det.\ Heun) & 0 / 1 (det.\ Heun) & 0 / 1 (det.\ Heun) \\
\quad $S_{\min}$ / $S_{\max}$ & - & - & - & - \\
\midrule
\multicolumn{5}{l}{\textit{Posterior sampling}} \\
\quad $\sigma_{\mathrm{obs}}$ (standardized) & \multicolumn{4}{c}{$5\times10^{-3}$} \\
\quad $\hat\Gamma$ & \multicolumn{4}{c}{$5\times10^{-4}$} \\
\quad Likelihood scale $\beta$ & \multicolumn{4}{c}{1} \\
\quad Gradient clip $c_{\mathrm{clip}}$ & \multicolumn{2}{c}{0.1} & \multicolumn{2}{c}{1} \\
\quad Ensemble members $E$ & \multicolumn{2}{c}{12} & \multicolumn{2}{c}{24} \\
\midrule
\multicolumn{5}{l}{\textit{Error predictor $P_\phi$}} \\
\quad Architecture & \multicolumn{2}{c}{conv.\ U-Net, base 64, $[1,2,4,8]$} & \multicolumn{2}{c}{as $D_\theta$, width 128} \\
\quad Parameters & \multicolumn{2}{c}{7.70\,M} & \multicolumn{2}{c}{1.05\,M} \\
\quad Training time & \multicolumn{2}{c}{1.7\,h on 12 GPUs} & \multicolumn{2}{c}{1.1\,h on 12 GPUs} \\
\bottomrule

\end{tabular}
\end{table}

\begin{table}[!htp]
\centering
\caption{Measured training cost of the single-step and multi-step models for the backward-facing step (25{,}000\,kimg each) and of the error predictor $P_\phi$, all on 12 GPUs. Wall-clock time and GPU-hours include training repeated after job restarts; throughput is given per curriculum stage $K$. Peak memory per GPU: allocated to tensors / reserved by the allocator.}
\label{tab:training_cost}
\footnotesize
\begin{tabular}{lcccc}
\toprule
recipe & wall-clock & GPU-hours & peak memory (GB) & throughput (kimg/h) \\
\midrule
single-step ($K=1$) & 58.5\,h & 702 & 10.6 / 15.8 & 432 \\
multi-step ($K=4$, curriculum) & 183.8\,h & 2{,}206 & 10.6 / 21.9 & 428 / 171 / 105 / 76 ($K{=}1$-$4$) \\
error predictor $P_\phi$ (100 epochs) & 1.1\,h & 13 & 3.6 & - \\
\bottomrule
\end{tabular}
\end{table}

\subsection{Boundary exclusion band and time-averaged profiles for the backward-facing step}\label{sec:boundary_band}
The excluded band of Section~\ref{sec:sensor_placement} is thin and runs along the top and bottom walls: for $g=5$ it excludes 16.9\,\% of the mesh nodes and is 2.8\,mm wide on average (at most 14\,mm); for $g=7$ it excludes 21.3\,\%. Apart from one outlet node, the step face, the upstream floor and the jagged inlet and outlet edges are not detected as boundaries, so most of their nodes stay eligible. Behind the step the band covers part of the region of large ensemble spread next to the floor, but most of the shear layer, where the informed strategies place their sensors (Figure~\ref{fig:bfs_plots/different_sampling}), stays eligible.

\paragraph{Time-averaged profiles.} Table~\ref{tab:bfs_profiles_stats} lists the errors of the time-averaged statistics of Section~\ref{sec:bfs_profiles} for eight runs, and Figure~\ref{fig:bfs_profiles_methods} adds the Reynolds-stress profiles of the EnKF runs and of the QR pivoting layout.

\begin{figure}[!htp]
    \centering
    \includegraphics[width=\linewidth]{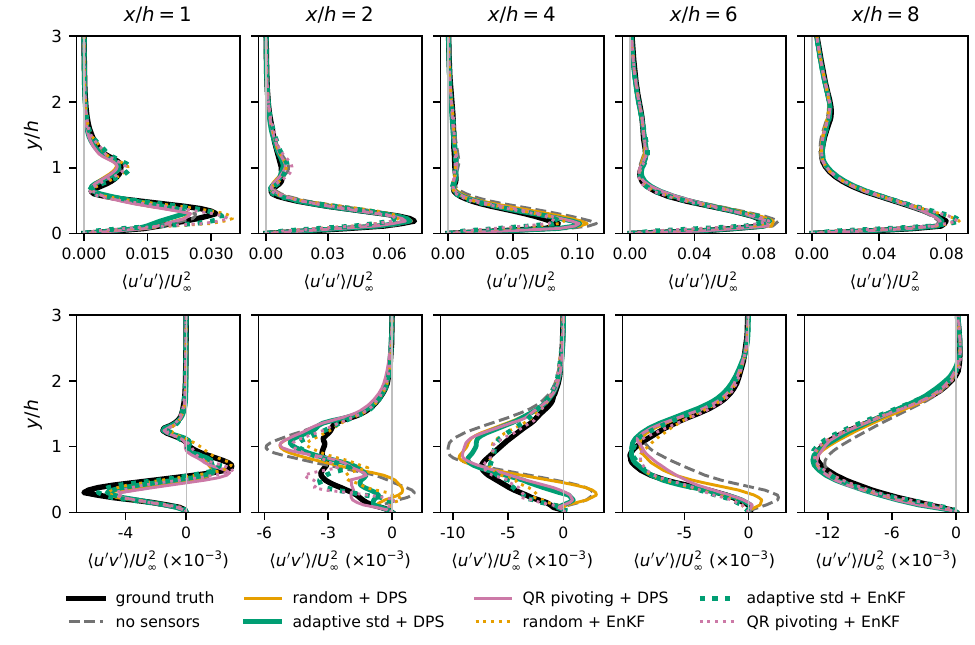}
    \caption{Time-averaged $\langle u'u'\rangle/U_\infty^2$ (top) and $\langle u'v'\rangle/U_\infty^2$ (bottom; axis in units of $10^{-3}$) for the ground truth, the rollout without sensors and the random, adaptive std and QR pivoting layouts with DPS (solid) and with the EnKF (dotted); stations and averaging as in Figure~\ref{fig:bfs_plots/mean_profiles}.}
    \label{fig:bfs_profiles_methods}
\end{figure}

\begin{table}[!htp]
\centering
\caption{Time-averaged statistics of Section~\ref{sec:bfs_profiles} ($s=50$, $g=5$ hops; averages over 40 forecast steps, all 24 members and the five initial conditions). Columns $\langle u\rangle$ to $\langle v'v'\rangle$: relative $L_2$ error $\|\phi-\phi_{\mathrm{truth}}\|_2/\|\phi_{\mathrm{truth}}\|_2$ of each statistic $\phi$ over the five station profiles ($x/h=1,2,4,6,8$) concatenated, against the ground truth averaged over the same steps. $X_r$: reattachment length, the last downstream position at which the time-averaged streamwise velocity $0.05\,h$ above the floor is negative (ground truth $7.19\,h$). -: not computed.}
\label{tab:bfs_profiles_stats}
\small
\begin{tabular}{lcccccc}
\toprule
run & $\langle u\rangle$ & $\langle v\rangle$ & $\langle u'u'\rangle$ & $\langle u'v'\rangle$ & $\langle v'v'\rangle$ & $X_r/h$ \\
\midrule
no sensors & 0.018 & 0.165 & 0.202 & 0.418 & 0.171 & 7.01 \\
random + DPS & 0.014 & 0.129 & 0.147 & 0.324 & 0.107 & 7.19 \\
predicted + DPS & 0.012 & 0.132 & 0.106 & 0.218 & 0.135 & 7.16 \\
adaptive std + DPS & 0.011 & 0.110 & 0.093 & 0.205 & 0.084 & 7.22 \\
QR pivoting + DPS & 0.011 & - & 0.108 & 0.257 & - & 7.25 \\
random + EnKF & 0.010 & - & 0.165 & 0.120 & - & 6.77 \\
adaptive std + EnKF & 0.010 & - & 0.176 & 0.097 & - & 6.65 \\
QR pivoting + EnKF & 0.010 & - & 0.180 & 0.107 & - & 6.65 \\
\bottomrule
\end{tabular}
\end{table}

\section{Additional results for the backward-facing step}\label{sec:bfs_additional}
This appendix collects supporting material for Section~\ref{sec:bfs}.
\subsection{Guidance strength}\label{sec:bfs_guidance}
Without the clip, the sampler diverges for every $\beta$ we tried (Table~\ref{tab:dps_sweep}). With the clip, raising $\beta$ from 1 to 3 helps both the adaptive std and the QR pivoting layouts; at a given $\beta$, the clip value changes the stable results by less than 2 points.

\begin{table}[!htp]
\centering
\caption{Guidance-strength ablation of DPS on one initial condition (test snapshot 868) over 100 forecast steps, with $s=50$, $g=5$ hops and 24 members: average RMSE reduction (\%) with the adaptive std sensors and the fixed QR layout for several likelihood scales $\beta$ and gradient clips $c_{\mathrm{clip}}$. These values are not comparable with Table~\ref{tab:bfs_baselines}, which averages five initial conditions over 40 steps. ``diverges'': the RMSE blows up within the first seven steps.}
\label{tab:dps_sweep}
\small
\begin{tabular}{llcc}
\toprule
$\beta$ & $c_{\mathrm{clip}}$ & adaptive std & QR pivoting \\
\midrule
1 (default) & 1 & 54.2 & 33.5 \\
1 & 2 & 55.8 & 33.4 \\
1 & 5 & 56.0 & 32.3 \\
2 & 2 & 62.6 & 36.3 \\
3 & 1 & \textbf{65.3} & 38.7 \\
3 & 2 & 64.9 & 38.3 \\
5 & 1 & 63.4 & \textbf{40.3} \\
10 & 5 & diverges & diverges \\
3, 100 & none & diverges & diverges \\
\bottomrule
\end{tabular}
\end{table}

\subsection{Fixed sensor layouts}\label{sec:bfs_fixed_layouts}
Figure~\ref{fig:bfs_fixed_layouts} shows the four fixed layouts compared in Sections~\ref{sec:bfs_fixed} and \ref{sec:bfs_baselines}.

\begin{figure}[!htp]
    \centering
    \includegraphics[width=\linewidth]{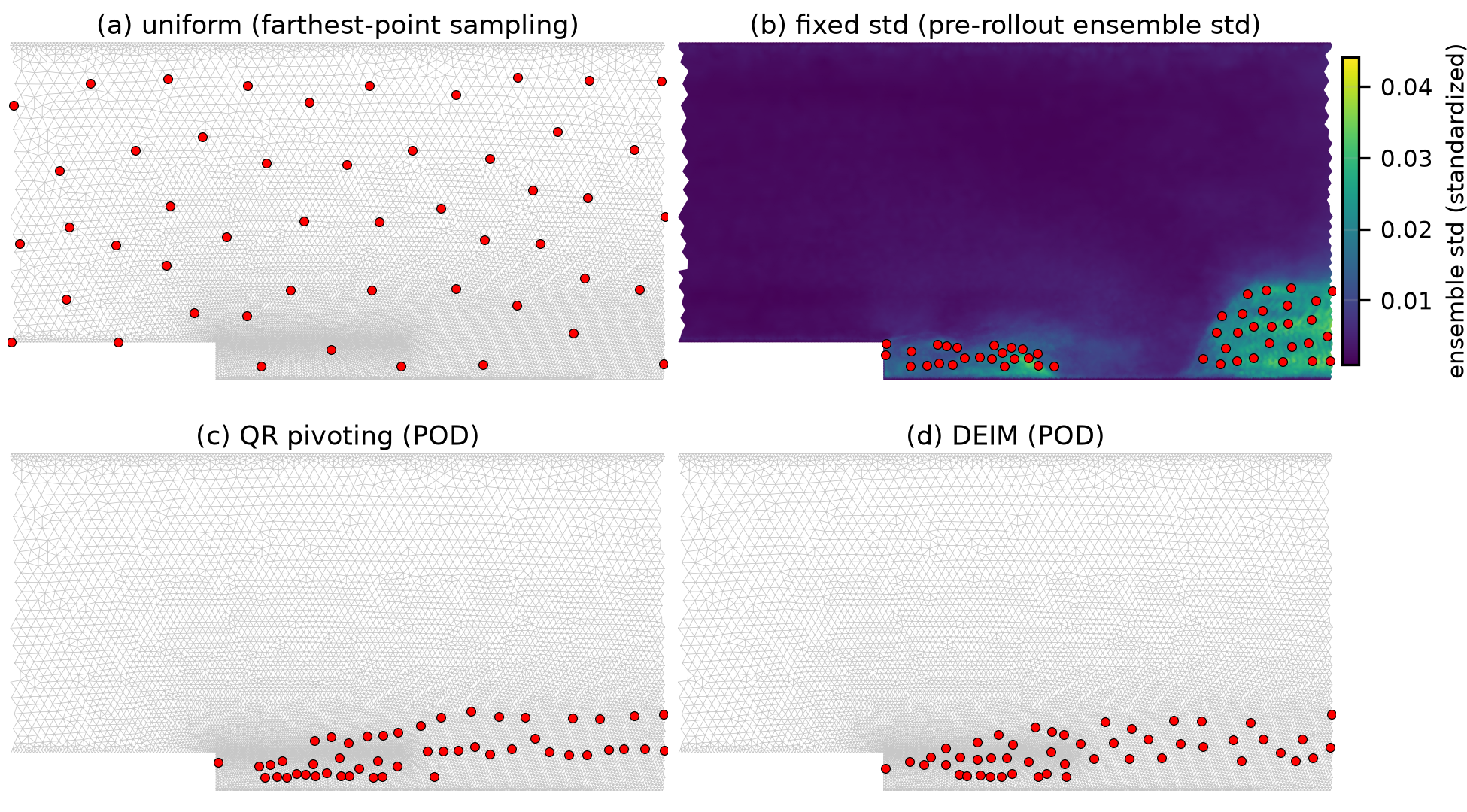}
    \caption{The four fixed layouts of $s=50$ sensors (red dots; $g=5$ hops) for the initial condition at test snapshot 902: (a) uniform, (b) fixed std, (c) QR pivoting and (d) DEIM. The colour in (b) is the selection score of the fixed std layout (Section~\ref{sec:bfs_fixed}); only this layout depends on the initial condition.}
    \label{fig:bfs_fixed_layouts}
\end{figure}

\subsection{EnKF settings, continuity projection and variants}\label{sec:bfs_enkf}\label{sec:bfs_continuity}
Table~\ref{tab:bfs_enkf} lists the settings of the EnKF of Section~\ref{sec:enkf} and of its continuity projection.

Before every EnKF update we project each forecast member,
\[
    \mathbf{x}\leftarrow\mathbf{x}-\mathbf{D}^\top\bigl(\mathbf{D}\mathbf{D}^\top+\mu_{\mathrm{p}}\mathbf{I}\bigr)^{-1}\mathbf{D}\mathbf{x},\qquad \mu_{\mathrm{p}}=\sigma_c^2/\sigma_b^2,
\]
with $\mathbf{D}$ the per-element divergence of the piecewise-linear (P1) interpolant on the mesh. This is the maximum-a-posteriori estimate for an isotropic Gaussian prior centred on the member, with standard deviation $\sigma_b$, and a Gaussian divergence residual with standard deviation $\sigma_c$. Even the true field has such a residual, because its P1 interpolant is not divergence-free element by element; $\sigma_c$ is the RMS of this residual. The no-slip wall nodes are held fixed. Because the members are projected before the update, the Kalman gain is formed from nearly divergence-free forecast anomalies.

\paragraph{EnKF variants.} We chose this configuration in a separate study on one initial condition that varied the projection (none, after or before the update), the localization length (0.01 or 0.03\,m), whether the wall nodes are held fixed, and $\sigma_{\mathrm{obs}}$ ($2\times10^{-3}$ to $10^{-2}$). Three problems appeared. Without the projection, the small-scale energy of the analysis is 22 times that of the truth. With the longer localization, the filter diverges with the adaptive std layout. With free wall nodes, the projection changes the velocity at the no-slip walls, which is unphysical.

\begin{table}[!htp]
\centering
\caption{EnKF settings for the backward-facing step; $\sigma_{\mathrm{obs}}$ and $\sigma_b$ are in standardized units.}
\label{tab:bfs_enkf}
\small
\begin{tabular}{ll}
\toprule
setting & value \\
\midrule
ensemble & the $E=24$ forecast members \\
forecast covariance & sample covariance with Gaspari-Cohn localization \\
localization length & $c=0.01$\,m, Euclidean distance (not graph hops); taper zero beyond $2c$ \\
inflation & none \\
observation error & $\mathbf{R}=\sigma_{\mathrm{obs}}^2\mathbf{I}$, $\sigma_{\mathrm{obs}}=5\times10^{-3}$ as for DPS \\
sensor layouts & shared by all members; the random layout is redrawn every step \\
adaptive std layout & from the spread of the previous analysis ensemble; random layout at $t=1$ \\
\midrule
continuity projection & every forecast member, before the update \\
prior standard deviation & $\sigma_b=0.03$ \\
residual standard deviation & $\sigma_c=102$\,s$^{-1}$ \\
held fixed & the 541 no-slip wall nodes \\
linear solver & Jacobi-preconditioned conjugate gradients, median 97 iterations \\
\bottomrule
\end{tabular}
\end{table}

\subsection{Noisy observations}\label{sec:bfs_noise}
Noisy sensor readings cost little accuracy if the assumed observation standard deviation $\sigma_{\mathrm{obs}}$ keeps the reference value of the noiseless runs, and much if it is matched to the noise. We repeat the adaptive std runs of Table~\ref{tab:bfs_baselines}, with DPS and with the EnKF, and add Gaussian noise to every sensor reading at every step. Its standard deviation is $\sigma_n=0.005$, 0.02 or 0.05 in standardized units, i.e.\ 0.07, 0.28 or 0.69\,m/s on $u$. The largest is about 2\,\% of the bulk velocity $U_\infty$. We compare the reference $\sigma_{\mathrm{obs}}=5\times10^{-3}$ with the matched $\sigma_{\mathrm{obs}}=\sigma_n$ (Table~\ref{tab:bfs_noise}, Figure~\ref{fig:bfs_noisy_obs}).

With the reference value, even the largest noise costs at most 2.7 points of average RMSE reduction. With the matched value, the losses are large: at $\sigma_n=0.05$ almost none of the average RMSE reduction remains (Table~\ref{tab:bfs_noise}, Figure~\ref{fig:bfs_noisy_obs}b). The matched value under-weights the sensors. The EnKF shows an under-dispersive ensemble (Section~\ref{sec:bfs_metrics}), which, with the true noise level, over-trusts the forecast. The ensemble filters often widen the spread by covariance inflation for this reason. In DPS, $\sigma_{\mathrm{obs}}$ is only a nominal weight, and raising it to $\sigma_n$ weakens the guidance.

\begin{figure}[!htp]
    \centering
    \includegraphics[width=\linewidth]{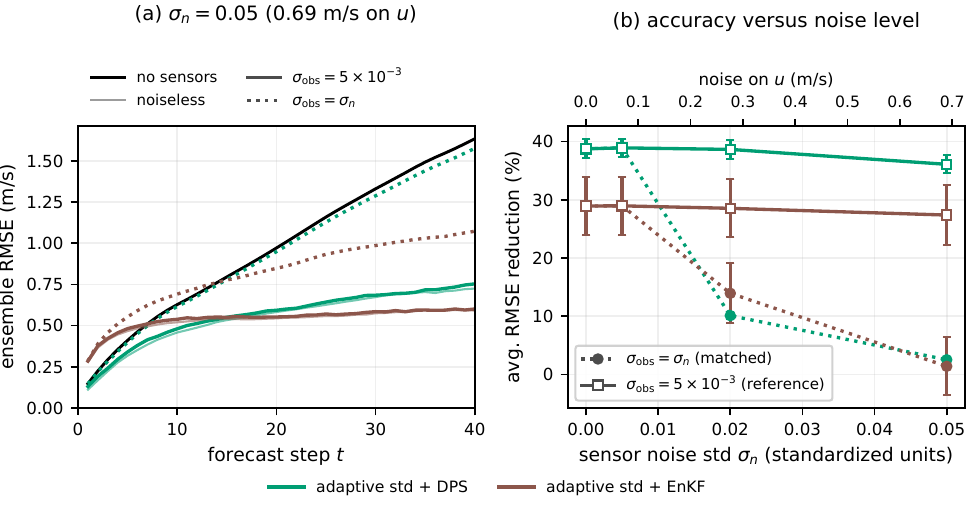}
    \caption{Noisy observations with adaptive std sensors, for DPS and for the EnKF (runs of Table~\ref{tab:bfs_noise}). (a) Ensemble RMSE versus forecast step, mean over the five initial conditions: without sensors (black), with noiseless readings (faint) and with $\sigma_n=0.05$ for the reference $\sigma_{\mathrm{obs}}=5\times10^{-3}$ (solid) and the matched $\sigma_{\mathrm{obs}}=\sigma_n$ (dotted). (b) Average RMSE reduction versus $\sigma_n$ (standardized units; upper axis in m/s on $u$) for the matched (filled circles, dotted) and the reference $\sigma_{\mathrm{obs}}$ (open squares, solid); error bars: $\pm$ one standard error over the five initial conditions.}
    \label{fig:bfs_noisy_obs}
\end{figure}

\begin{table}[!htp]
\centering
\caption{Average RMSE reduction (\%, mean $\pm$ standard error over the five initial conditions) of the adaptive std runs of Table~\ref{tab:bfs_baselines} with Gaussian sensor noise of standard deviation $\sigma_n$ (standardized units). At $\sigma_n=0.005$ the matched and the reference $\sigma_{\mathrm{obs}}$ coincide.}
\label{tab:bfs_noise}
\small
\begin{tabular}{lcccccc}
\toprule
update & noiseless & \multicolumn{3}{c}{reference $\sigma_{\mathrm{obs}}=5\times10^{-3}$} & \multicolumn{2}{c}{matched $\sigma_{\mathrm{obs}}=\sigma_n$} \\
\cmidrule(lr){3-5}\cmidrule(lr){6-7}
$\sigma_n$ & 0 & 0.005 & 0.02 & 0.05 & 0.02 & 0.05 \\
\midrule
DPS & 38.8\,$\pm$\,1.6 & 38.9\,$\pm$\,1.5 & 38.7\,$\pm$\,1.6 & 36.1\,$\pm$\,1.6 & 10.1\,$\pm$\,0.5 & 2.5\,$\pm$\,0.2 \\
EnKF & 28.9\,$\pm$\,5.0 & 29.0\,$\pm$\,5.0 & 28.5\,$\pm$\,5.0 & 27.4\,$\pm$\,5.2 & 13.9\,$\pm$\,5.2 & 1.4\,$\pm$\,5.0 \\
\bottomrule
\end{tabular}
\end{table}

\subsection{Summary metrics}\label{sec:bfs_metrics_table}
Table~\ref{tab:bfs_metrics} lists the metrics of Section~\ref{sec:bfs_metrics} for eight runs.

\begin{table}[!htp]
\centering
\caption{Metrics of Section~\ref{sec:bfs_metrics} over 40 forecast steps, mean over the five initial conditions ($s=50$, $g=5$ hops, 24 members). In every run the anomaly correlation of the ensemble mean with the truth (anomalies with respect to the per-node time mean of the 924 retained snapshots) exceeds 0.6 at all 40 steps. CRPS (m/s): ensemble estimator $\frac{1}{E}\sum_e|x_e-y|-\frac{1}{2E^2}\sum_{e,e'}|x_e-x_{e'}|$ for each state entry (members $x_e$, truth $y$), averaged over nodes, both velocity components and steps. Spread/skill: ensemble standard deviation over the RMSE of the ensemble mean, averaged over steps. High-$k$ ratio: predicted over true energy in the upper half of the resolved wavenumbers of the streamwise spectrum of $u$ along three lines through the shear layer ($y/h=0.5$, 1 and 1.5), from six of the 24 members at every fifth step from 20 to 40. $e_1$ (m/s): one-step ensemble RMSE. Div.\ RMS (s$^{-1}$): per-element divergence RMS of the P1 velocity field (\appref{sec:bfs_continuity}); about 100\,s$^{-1}$ for the ground truth.}
\label{tab:bfs_metrics}
\footnotesize
\begin{tabular}{lccccc}
\toprule
run & CRPS & spread/skill & high-$k$ ratio & $e_1$ & div.\ RMS \\
\midrule
no sensors & 0.53 & 0.09 & 1.59 & 0.148 & 136 \\
random + DPS & 0.34 & 0.24 & 2.57 & 0.112 & 142 \\
predicted + DPS & 0.34 & 0.16 & 2.62 & 0.082 & 134 \\
adaptive std + DPS & 0.30 & 0.17 & 2.87 & 0.112 & 137 \\
QR pivoting + DPS & 0.35 & 0.13 & 2.78 & 0.094 & 141 \\
random + EnKF & 0.28 & 0.06 & 0.76 & 0.281 & 46 \\
adaptive std + EnKF & 0.28 & 0.05 & 0.89 & 0.281 & 46 \\
QR pivoting + EnKF & 0.32 & 0.05 & 0.96 & 0.278 & 45 \\
\bottomrule
\end{tabular}
\end{table}

\section{Single-step versus multi-step training}\label{sec:multistep_app}
The multi-step objective trains the model on its own rollouts. Given the current state $\mathbf{x}^{(0)}$ and the targets $\mathbf{x}^{(1)},\dots,\mathbf{x}^{(K)}$, step $k$ is conditioned not on the true state $\mathbf{x}^{(k-1)}$ but on the model's own estimate of it, a single denoiser evaluation from pure noise,
\begin{equation}\label{eq:one_step_prediction}
    \hat{\mathbf{x}}^{(k)} = D_\theta\!\left(\sigma_{\max}\boldsymbol{\epsilon}^{(k)};\, \sigma_{\max},\, \hat{\mathbf{x}}^{(k-1)}\right), \qquad
    \boldsymbol{\epsilon}^{(k)}\sim\mathcal{N}(\mathbf{0},\mathbf{I}), \qquad \hat{\mathbf{x}}^{(0)} = \mathbf{x}^{(0)},
\end{equation}
with $\sigma_{\max}=80$, and the single-step objective of Section~\ref{sec:training} is averaged over the $K$ steps,
\begin{equation}\label{eq:multistep_loss}
    \mathcal{L}_{K} = \frac{1}{\sum_{k=1}^{K} w_k}\sum_{k=1}^{K}\, w_k\,
    \mathbb{E}_{\sigma^{(k)},\,\mathbf{n}^{(k)}}\!\left[
        \lambda\!\left(\sigma^{(k)}\right)
        \left\|
            D_\theta\!\left(\mathbf{x}^{(k)} + \mathbf{n}^{(k)};\, \sigma^{(k)},\, \hat{\mathbf{x}}^{(k-1)}\right) - \mathbf{x}^{(k)}
        \right\|_F^2
    \right],
\end{equation}
with $w_k = 0.9^{\,k-1}$. The fed-back states $\hat{\mathbf{x}}^{(k-1)}$ are detached, so that no gradient flows through them (the pushforward trick \cite{brandstetter2022message}), and $K$ grows to 4 through a curriculum (Table~\ref{tab:combined_hyperparams}). Each optimizer step thus evaluates the network $2K-1$ times: $K$ times for the loss terms and $K-1$ times, without gradients, for the fed-back states.

\paragraph{Results.} Multi-step training brings little or no gain (Table~\ref{tab:single_vs_multi}). In both cases the multi-step model is less accurate over the first forecast steps. For the backward-facing step it lowers the ensemble RMSE averaged over the 100 steps by 10\,\%, but its training takes about three times as long (Table~\ref{tab:training_cost}). For the turbulence case, at about the same training length (20{,}112 and 19{,}604\,kimg), the multi-step model is worse than the single-step model on every initial condition during the first ten steps and about as accurate once the forecasts have decorrelated from the truth.

\begin{table}[!htp]
\centering
\caption{Multi-step versus single-step model: unconditional rollouts (Figures~\ref{fig:bfs_plots/single_vs_multi}-\ref{fig:hit_single_vs_multi_error}). Errors and training time are relative to the single-step model. Backward-facing step: means over the five initial conditions (range in parentheses).}
\label{tab:single_vs_multi}
\footnotesize
\begin{tabular}{lcc}
\toprule
 & backward-facing step & two-dimensional turbulence \\
\midrule
training (kimg), single / multi-step & 25{,}000 / 25{,}000 & 20{,}112 / 19{,}604 \\
initial conditions & 5 test snapshots (838-878) & 12 snapshots of the test segment \\
members, forecast steps & 24, 100 & 4, 50 \\
error measure & ensemble RMSE & MAE \\
early error & 2.5 times (step 1), $+23$\,\% (step 10) & 1.5-1.7 times (steps 1-10) \\
mean error over all steps & $-10$\,\% ($0$ to $-19$\,\%) & $+4$\,\% \\
training time & about 3 times (Table~\ref{tab:training_cost}) & about 3 times \\
\bottomrule
\end{tabular}
\end{table}

\paragraph{Why the gain is small.} By Tweedie's formula, the fed-back state of Equation~\eqref{eq:one_step_prediction} is an estimate of the conditional mean, an average over the possible next states rather than a sample. The later loss terms are therefore conditioned on means, whereas at inference the later steps are conditioned on the model's own full samples, which it never sees in training. Since the network cannot tell a mean from a true state, training on these terms plausibly also reduces its one-step accuracy. Conditioning on full samples would remove the mismatch, but each fed-back state would then need a full 50-step sampler run (99 network evaluations): $(K-1)\cdot 99+K=301$ instead of $2K-1=7$ evaluations per optimizer step for $K=4$.

The assimilation results of Section~\ref{sec:results} use the multi-step model for the backward-facing step and the single-step model of Table~\ref{tab:combined_hyperparams} (30{,}000\,kimg) for the turbulence case; the method itself does not depend on how the forecast model was trained.

\begin{figure}[!htp]
    \centering
    \includegraphics[width=\linewidth]{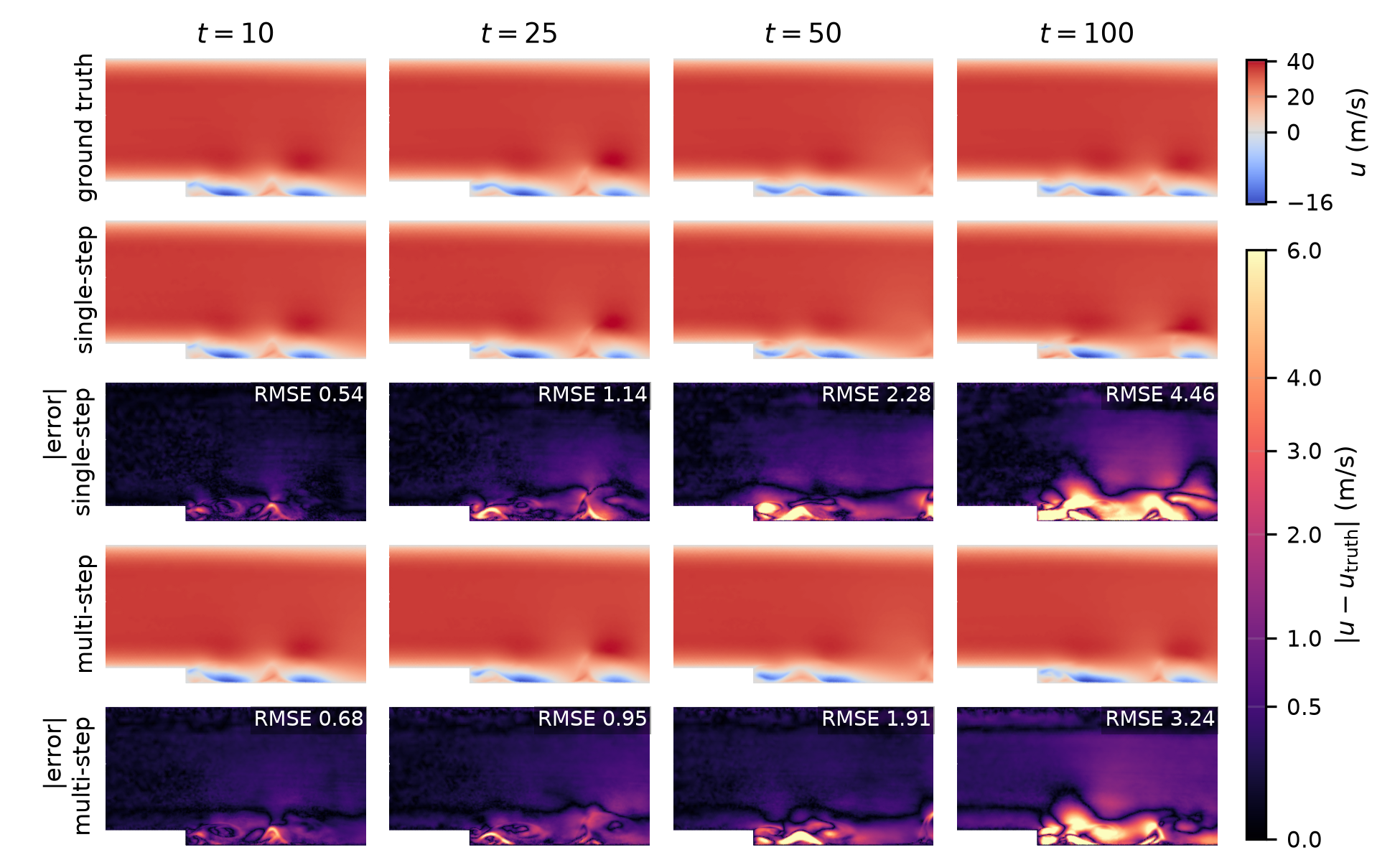}
    \caption{Streamwise velocity $u$ (m/s) of the ground truth and of one ensemble member of the single-step and the multi-step model, each followed by its absolute error $|u-u_{\mathrm{truth}}|$ (common power-law colour scale; each panel gives the member's RMSE of $u$ in m/s), at forecast steps $t=10,25,50,100$ of an unconditional rollout from test snapshot 868. Both models err mostly in the shear layer and the recirculation zone; the single-step error starts lower but grows faster.}
    \label{fig:bfs_plots/single_vs_multi}
\end{figure}
\begin{figure}[!htp]
    \centering
    \includegraphics[width=0.62\linewidth]{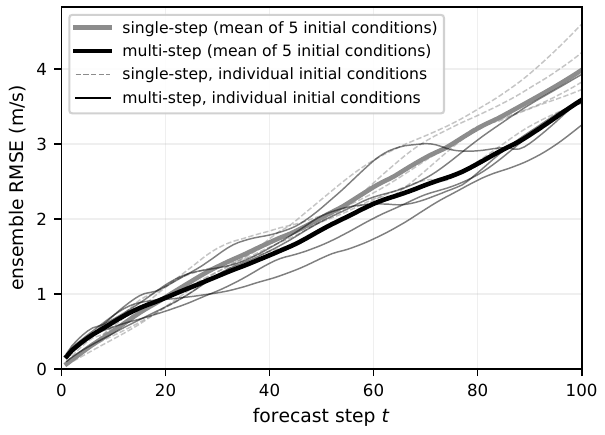}
    \caption{Ensemble RMSE (24 members) of unconditional rollouts of the single-step and multi-step models from five test initial conditions (snapshots 838, 848, 858, 868 and 878).}
    \label{fig:bfs_plots/single_vs_multi_error}
\end{figure}
\begin{figure}[!htp]
    \centering
    \includegraphics[width=\linewidth]{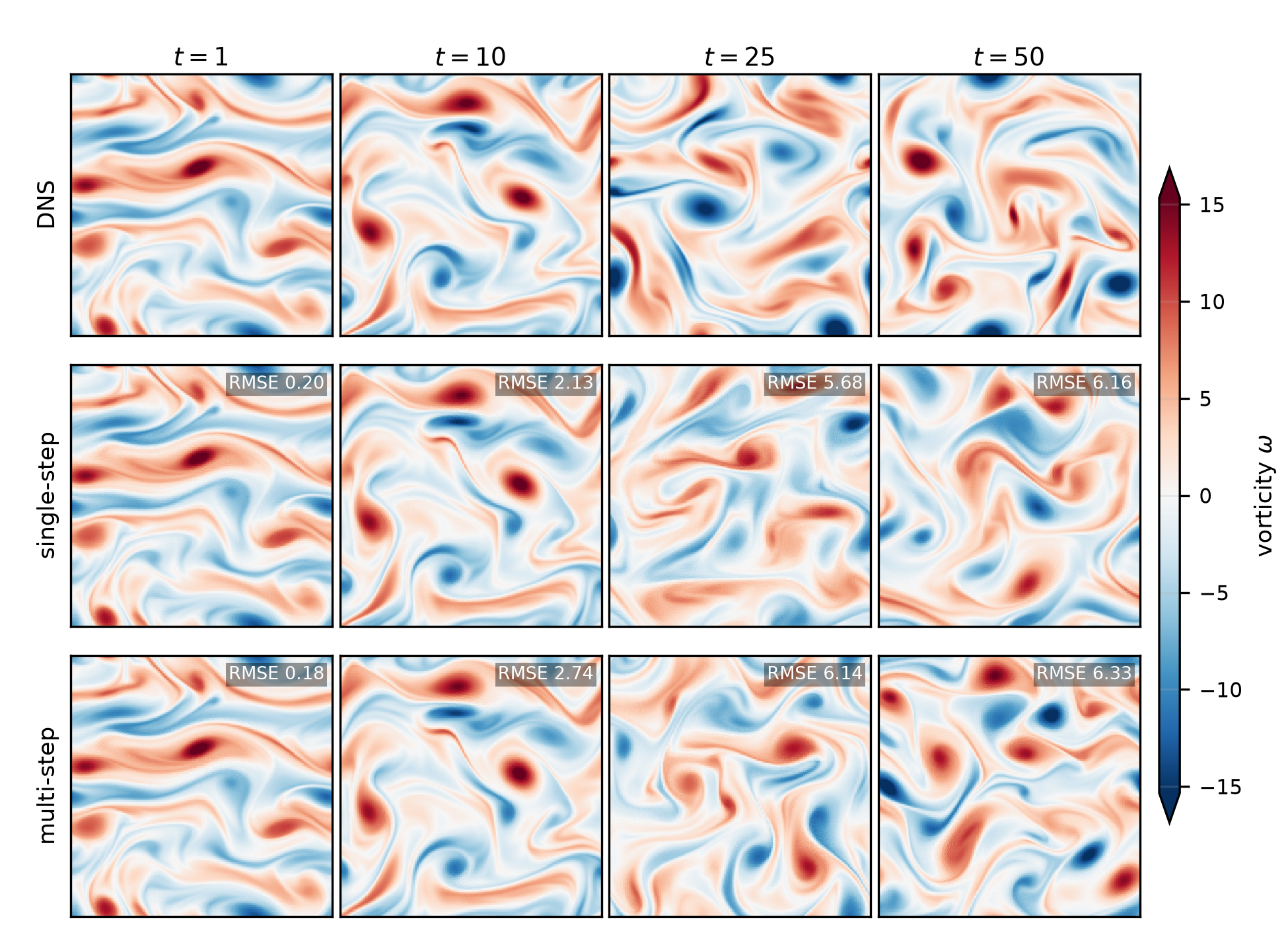}
    \caption{Two-dimensional turbulence: vorticity $\omega$ of the DNS and of one member of the single-step (20{,}112\,kimg) and the multi-step model (19{,}604\,kimg) at forecast steps $t=1,10,25,50$ of an unconditional rollout from the first test initial condition; each panel gives the vorticity RMSE of that member.}
    \label{fig:hit_single_vs_multi}
\end{figure}
\begin{figure}[!htp]
    \centering
    \includegraphics[width=\linewidth]{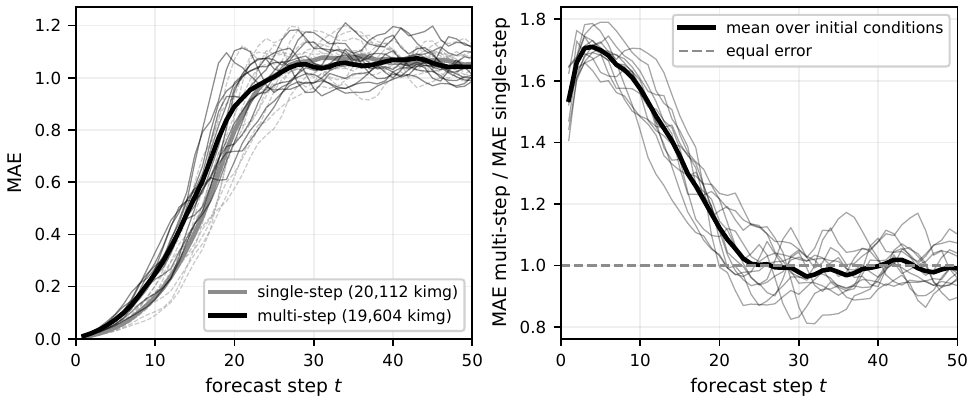}
    \caption{Two-dimensional turbulence: MAE (mean over 4 members, all grid points and both velocity components) of unconditional rollouts from 12 snapshots of the test segment. Left: individual initial conditions (thin; dashed: single-step model) and their mean (thick). Right: ratio of the multi-step to the single-step MAE per initial condition. Both models start every member from the same initial noise.}
    \label{fig:hit_single_vs_multi_error}
\end{figure}

\end{document}